\documentclass[twoside,11pt]{article}

\usepackage{jmlr2e}

\usepackage{graphicx}
\usepackage{xspace} %
\usepackage{amsmath}
\usepackage{amssymb}
\usepackage{amsfonts}
\usepackage{amsbsy} %
\usepackage{booktabs} %

\newenvironment{thm}{\begin{theorem}}{\end{theorem}}
\newenvironment{prop}{\begin{proposition}}{\end{proposition}}
\newenvironment{lem}{\begin{lemma}}{\end{lemma}}
\newenvironment{cor}{\begin{corollary}}{\end{corollary}}

\newenvironment{defn}{\begin{definition} \rm}{\end{definition}}
\newenvironment{ex}{\begin{example} \rm}{\end{example}}

\newenvironment{AppProof}[1]{{\noindent \bf Proof of #1}\quad}{\begin{flushright}$\square$\end{flushright}}

\newtheorem{asm}{Assumption}

\DeclareMathOperator*{\argmin}{argmin}
\DeclareMathOperator*{\sgn}{sgn}

\def\bs#1{\boldsymbol#1}
\newcommand{\mat}[2]{{\rm M}(#1,#2)}                                  %
\newcommand{\pd}[2]{\frac{\partial #1}{\partial #2}}                  %
\newcommand{\pds}[3]{\frac{\partial^2 #1}{\partial #2 \partial #3}}   %
\newcommand{\pdseta}[3]{\frac{\partial^2 #1}{\partial \eta_{#2} \partial \eta_{#3}}} %
\newcommand{\norm}[1]{\| #1 \|}
\newcommand{\spec}[1]{{\rm Spec}(#1)}                 %
\newcommand{\specr}[1]{\rho(#1)}                      %

\newcommand{\E}[2]{{\rm E}_{#1}[#2]}                  %
\newcommand{\cov}[3]{{\rm Cov}_{#1}[#2,#3]}
\newcommand{\corr}[3]{{\rm Cor}_{#1}[#2,#3]}           %
\newcommand{\var}[2]{{\rm Var}_{#1}[#2]}

\newcommand{\inp}[2]{\langle#1,#2\rangle}             %

\newcommand{\diag}{{\rm diag}}

\newcommand{\ug}[1]{{\rm G}_{#1}}                     %

\newcommand{\tr}{{\rm tr}}                   %

\newcommand{\beliefs}{ \{b_{\alpha}(x_{\alpha}), b_i(x_i) \}_{\alpha \in F, i \in V}  }  %
\newcommand{\beliefsw}{ \{b_{\alpha}(x_{\alpha}), b_i(x_i) \}  }  %
\newcommand{\thetasw}{ \{\fa{\theta}, \theta_i \}  }                                 %

\newcommand{\fgdefn}{H=(V \cup F,\vec{E})}   %

\newcommand{\etea}{e' \rightharpoonup e}

\newcommand{\ete}[2]{#1 \rightharpoonup #2}
\newcommand{\vfe}{\mathfrak{X}(\vec{E})}     %
\newcommand{\vfv}{\mathfrak{X}(V)}           %
\newcommand{\matmu}{\mathcal{M}(\bs{u})}     %
\newcommand{\matmc}{\mathcal{M}(\bs{c})}     %

\newcommand{\ed}[2]{#1 \rightarrow #2}       %
\newcommand{\edai}{\alpha \rightarrow i}

\newcommand{\edbi}{\beta \rightarrow i}
\newcommand{\edbj}{\beta \rightarrow j}
\newcommand{\edgi}{\gamma \rightarrow i}

\newcommand{\edbione}{\beta \rightarrow i_{1}}
\newcommand{\edbik}{\beta \rightarrow i_{k}}
\newcommand{\edij}{i \rightarrow j}
\newcommand{\edji}{j \rightarrow i}

\newcommand{\fai}{\{i_{1},\ldots,i_{d_{\alpha}}\}} %
\newcommand{\prodv}{\prod_{i \in V}}
\newcommand{\prodf}{\prod_{\alpha \in F}}

\newcommand{\Hesse}{\nabla^{2}}    %

\newcommand{\bsu}{\bs{u}}

\newcommand{\bsx}{\bs{x}}

\newcommand{\bsphi}{\bs{\phi}}
\newcommand{\bstheta}{\bs{\theta}}
\newcommand{\bseta}{\bs{\eta}}
\newcommand{\bsmu}{\bs{\mu}}

\newcommand\Bfe{Bethe free energy\xspace}
\newcommand\LBP{Loopy Belief Propagation\xspace}
\newcommand\lbp{loopy belief propagation\xspace}
\newcommand\BP{Belief Propagation\xspace}

\newcommand\ifa{inference family\xspace}
\newcommand\ifas{inference families\xspace}
\newcommand\Ifa{Inference family\xspace}
\newcommand\epara{expectation parameter\xspace}
\newcommand\eparas{expectation parameters\xspace}
\newcommand\npara{natural parameter\xspace}
\newcommand\nparas{natural parameters\xspace}

\newcommand\Bzf{Bethe-zeta formula\xspace}
\newcommand\IB{Ihara-Bass\xspace}
\newcommand\ccm{correlation coefficient matrix\xspace}
\newcommand\ccms{correlation coefficient matrices\xspace}

\newcommand{\fa}[1]{{{#1}_{\alpha}}}     %

\newcommand{\pa}[1]{{{#1}_{\langle \alpha \rangle}}}   %
\newcommand{\pb}[1]{{{#1}_{\langle \beta \rangle}}}
\newcommand{\pij}[1]{{{#1}_{\langle i,j \rangle}}}
\newcommand{\va}[2]{{#1}_{\alpha : #2}}

\jmlrheading{}{2011}{}{}{}{Yusuke Watanabe and Kenji Fukumizu}

\ShortHeadings{LBP, BFE and Graph Zeta Function}{Watanabe and Fukumizu}
\firstpageno{1}

\begin{document}

\title{Loopy Belief Propagation, Bethe Free Energy and\\ Graph Zeta Function}

\author{\name Yusuke Watanabe \email watay@ism.ac.jp \\
        \addr The Institute of Statistical Mathematics \\
        10-3 Midori-cho, Tachikawa, Tokyo 190-8562, Japan
       \AND
       \name Kenji Fukumizu \email fukumizu@ism.ac.jp \\
       \addr  The Institute of Statistical Mathematics \\
        10-3 Midori-cho, Tachikawa, Tokyo 190-8562, Japan}

\editor{}

\maketitle

\begin{abstract}%
We propose a new approach to the theoretical analysis of
Loopy Belief Propagation (LBP) and the Bethe free energy (BFE)
by establishing a formula to connect LBP and BFE with a graph zeta function.
The proposed approach is applicable to a wide class of models including multinomial and Gaussian types.
The connection derives a number of new theoretical results on LBP and BFE.
This paper focuses two of such topics.  One is the analysis of the region where the Hessian of
the Bethe free energy is positive definite, which derives 
the non-convexity of BFE for graphs with multiple cycles,
and a condition of convexity on a restricted set.
This analysis also gives a new condition for the uniqueness of the LBP fixed point.
The other result is to clarify the relation between the local stability of a fixed point of
LBP and local minima of the BFE, which implies, for example, 
that a locally stable fixed point of the Gaussian LBP is a local minimum of the Gaussian Bethe free energy.
\end{abstract}


\section{Introduction}
Probability density functions that have ``local'' factorization structures, called {\it graphical models},
constitute fundamentals in many fields.
In the fields of statistics, artificial intelligence and machine learning, for example, 
graphical modeling has been a powerful tool for representing our prior knowledge and modeling hidden structures
of problems \citep{Wgraphical,Pearl,Jlearning}.
Other examples are found in statistical physics, %
coding theory, %
and combinatorial optimizations \citep{Pcluster,Turbo,MPZanalytic}.
Typically, such probability distributions are derived from
random variables that only have local interactions/constraints.
This factorization structure is clearly visualized by a graph, called factor graph.

Since the {\it inference problems} on graphical models, such as computation of marginal/conditional density functions and partition functions,
are in general intractable for large graphs, %
{\it \LBP } (LBP) has been proposed as an efficient approximation method applicable to any graph-structured density functions.
Originally, {\it Belief Propagation} (BP) algorithm was proposed by \citet{Pearl} to compute exactly the marginals for 
tree-structured graphical models.
This algorithm passes ``messages'' between vertices of the graph until
all information of the graphical model is distributed throughout the graph.
Some researchers have found that LBP, an extended use of BP for graphs with cycles, shows good approximation with high potential applicability
\citep{MWJempiricalstudy,Turbo}.
After the proposal, many extensions and variants have been studied  \citep{YFWGBP,Sudderth2002,WJWMAP}
and have been applied successfully to many problems, including coding theory, 
image processing, sensor network localization and compressive sensing
\citep{ihler2005nonparametric,Baron2008}.

On the theoretical side, a significant number of studies have been carried out by many authors in this decade. %
One theoretical challenge of LBP is that the algorithm may have many fixed points;
the uniqueness is generally guaranteed only for trees and one-cycle graphs \citep{W1loop}.
The LBP fixed points are the solutions of a nonlinear equation associated with the graph, and
the structure of the equation is more complicated as the number of cycle is larger.
Regarding this problem, a notable result is the variational interpretation of LBP; it shows that
the LBP fixed points are the local minima of the Bethe free energy \citep{YFWGBP,YFWconstructing}.
This suggests that the behavior of LBP is more complex with non-convexity of the Bethe free energy.
Another difficulty of LBP is that the algorithm does not necessarily converge and sometimes shows oscillatory
behaviors.
Concerning the multinomial model (also known as discrete variable model), \citet{MKsufficient} and \citet{IFW} give sufficient
conditions for the convergence in terms of the spectral radius of a certain matrix related to the
graph.
\citet{TJgibbsmeasure} also derive a sufficient condition for convergence, interpreting the convergence as
the uniqueness of the Gibbs measure on the universal covering tree.

The purpose of this paper is to provide a novel discrete geometric approach to analysis of the LBP algorithm.
The starting point of our study is a question:
``How are the behaviors of the LBP algorithm affected by the geometry of the graph?''
If the graph is a tree (L)BP works nicely; it terminates in a finite step at the unique fixed point and gives the exact marginals.
If the graph has only one cycle it also works appropriately;
the algorithm converges to the unique fixed point and
finds the MPM (Maximum Posterior Marginal) assignment in binary variable cases \citep{W1loop}.
Additionally, the \Bfe function is convex in these cases \citep{PAstat}.
Existence of multiple cycles, however, breaks down these nice properties.
There have not been many researches that elucidate the effects of cycles on LBP in detail beyond ``tree or non-tree'' classification.
While a notable exception is the walk-sum analysis by \citet{JMWwalk} and \cite{MJWwalk}, it is limited to the Gaussian case.

This paper proposes a method based on a new connection between LBP, \Bfe, and a graph zeta function. 
Graph zeta functions, originally introduced by \citet{Idiscrete},
are popular graph characteristics defined by the products over the  prime cycles.
We capture the effects of cycles on LBP and \Bfe by establishing a novel formula, called \Bzf, which connects
the Hessian of the Bethe free energy with the graph zeta function.
To derive the formula, we extend the definition of existing graph zeta functions and related Ihara-Bass formula \citep{STzeta1,Bass}.

Our discovery of the connection, including the \Bzf, 
derives new ways of analyzing LBP and the \Bfe function taking the graph geometry into account.  
It is applicable to a wide class of graphical models defined by ``marginally closed'' exponential families, which include multinomial and Gaussian models. 
This paper discusses two examples of such analysis: one is the positive definiteness of Hessian for the Bethe free energy, 
and the other is local stability of the LBP dynamics.

First, based on the connection, we derive conditions that the Hessian of Bethe free energy function is positive definite.
As already discussed, analysis of the Bethe free energy is important for theoretical understanding of the complex behavior of LBP.
As the fundamentals, we consider the local properties of the Bethe free energy by elucidating the positive definiteness of its Hessian, 
while there are many studies on modifications and convexifications of the \Bfe function \citep{WHfractional,WJWtree,WYM}.
The direct consequence of our analysis is a sufficient condition of the uniqueness of the LBP fixed point, 
which is derived by giving a condition of global convexity.
In discussing the positive definiteness, 
we consider two defining domains of the Bethe free energy: one is given by the locally consistent pseudomarginals, 
and the other is a more restricted set conditioned by the compatibility functions of given graphical model.  
The beliefs given by LBP always lie in the latter domain.
We show that, when considered in the former domain, 
the necessary and sufficient condition for the Hessian to be positive definite is that the underlying factor graph has no more than one cycle.  
We also give a sufficient condition of the convexity of Bethe free energy on the latter domain, 
which implies the uniqueness of the LBP fixed point.  
By numerical examples, we demonstrate that our new uniqueness condition covers a wider region than the one given by \citet{MKsufficient} for the examples.

In the second application, we clarify a relation between the local structure of the Bethe free energy function
and the local stability of a LBP fixed point.
Such a relation is not necessarily obvious, since LBP is not derived as the gradient descent of the Bethe free energy.
In this line of studies, for multinomial models \citet{Hstable} shows that a locally stable fixed
point of LBP is a local minimum of the Bethe free energy.
We give conditions of the local stability of LBP and the positive definiteness of the Bethe
free energy in terms of the eigenvalues of a matrix that appears in the graph zeta function.
As a consequence, the result by Heskes is extended to a wider class including Gaussian distributions.

This paper is organized as follows.
In section 2, we introduce graphical models, LBP and the Bethe free energy as preliminaries.
We formulate the setting in terms of exponential families.
Section 3 includes the definition of a new class of graph zeta function, the extension of Ihara-Bass formula, and related results.
Using these results, Section 4 shows the fundamental results of this paper, 
\Bzf and positive definiteness condition, in Theorems~\ref{thm:BZ} and \ref{thm:positive}.
Section 5 derives a positive-definite region of the Bethe free energy function, and discusses convexity.
In section 6, we elucidate the relations between the stability of LBP
and the local structure of the \Bfe at LBP fixed points.
Section 7 includes discussion and concluding remarks.
Proofs omitted from the main body of the paper are given in the appendices.

\section{Preliminaries}

In this section we summarize a background of graphical models and LBP.
In Subsection~\ref{sec:prob} we introduce graphical models in terms of hypergraphs.
Subsection~\ref{sec:LBP} introduces LBP algorithm.
The Bethe free energy, which provides alternative language for formulating LBP algorithm,
is discussed in Subsection~\ref{sec:BFE}.

\subsection{Graphical models}\label{sec:prob}
We begin with basic definitions of hypergraphs because the associated structures with
graphical models are, precisely speaking, hypergraphs.

An ordinary {\it graph} $G=(V,E)$ consists of the vertex set $V$ joined by edges of $E$.
Generalizing the notion of graphs, hypergraphs are defined as follows.
A {\it hypergraph} $H=(V,F)$ consists of a set of {\it vertices} $V$ and a set of {\it hyperedges} $F$.
A hyperedge is a non-empty subset of $V$.
For any vertex $i \in V$, the {\it neighbors} of $i$ is defined by
$N_i:=\{\alpha \in F | i \in \alpha \}$.
Similarly, for any hyperedge $\alpha \in F$, the neighbors of $\alpha$ is defined by
$N_{\alpha}:=\{ i \in V | i \in \alpha \}=\alpha$.
The {\it degrees} of $i$ and $\alpha$ are given by
$d_i:=|N_i|$ and $d_{\alpha}:=|N_{\alpha}|=|\alpha|$, respectively.
If all the degrees of hyperedges are two, then the hypergraph is naturally identified with an ordinary graph.

In order to describe the message passing algorithm in Subsection \ref{sec:basicLBP},
it is convenient to identify a relation $i \in \alpha$ with a directed edge $\edai$.
For example, let
$H=(\{1,2,3,4\},\{\alpha_1,\alpha_2,\alpha_3\})$,
where $\alpha_1=\{1,2\}$, $\alpha_2=\{1,2,3,4\}$ and $\alpha_3=\{4\}$;
this hypergraph is shown as a directed graph in Fig.~\ref{fig:Dhypergraph}.
Explicitly writing the set of directed edges $\vec{E}$, a hypergraph $H$ is also denoted by $\fgdefn$.
Note that, forgetting the edge directions, $H$ is also represented as a bipartite graph
(Fig.~\ref{fig:Bhypergraph}).

We define basic notions of hypergraphs via its corresponding bipartite graphs.
A hypergraph $H$ is {\it connected} (resp. {\it tree}) if the corresponding bipartite graph is
connected (resp. tree).
In the same way, the number of connected components (resp. nullity) of $H$ is defined and denoted by $k(H)$ (resp. $n(H)$).
Therefore, $n(H):=|V|+|F|-|\vec{E}|$ and a hypergraph $H$ is a tree if and only if $n(H)=0$ and $k(H)=1$.

\begin{figure}
\begin{minipage}{.5\linewidth}
\begin{center}
\vspace{1mm}
\includegraphics[scale=0.3]{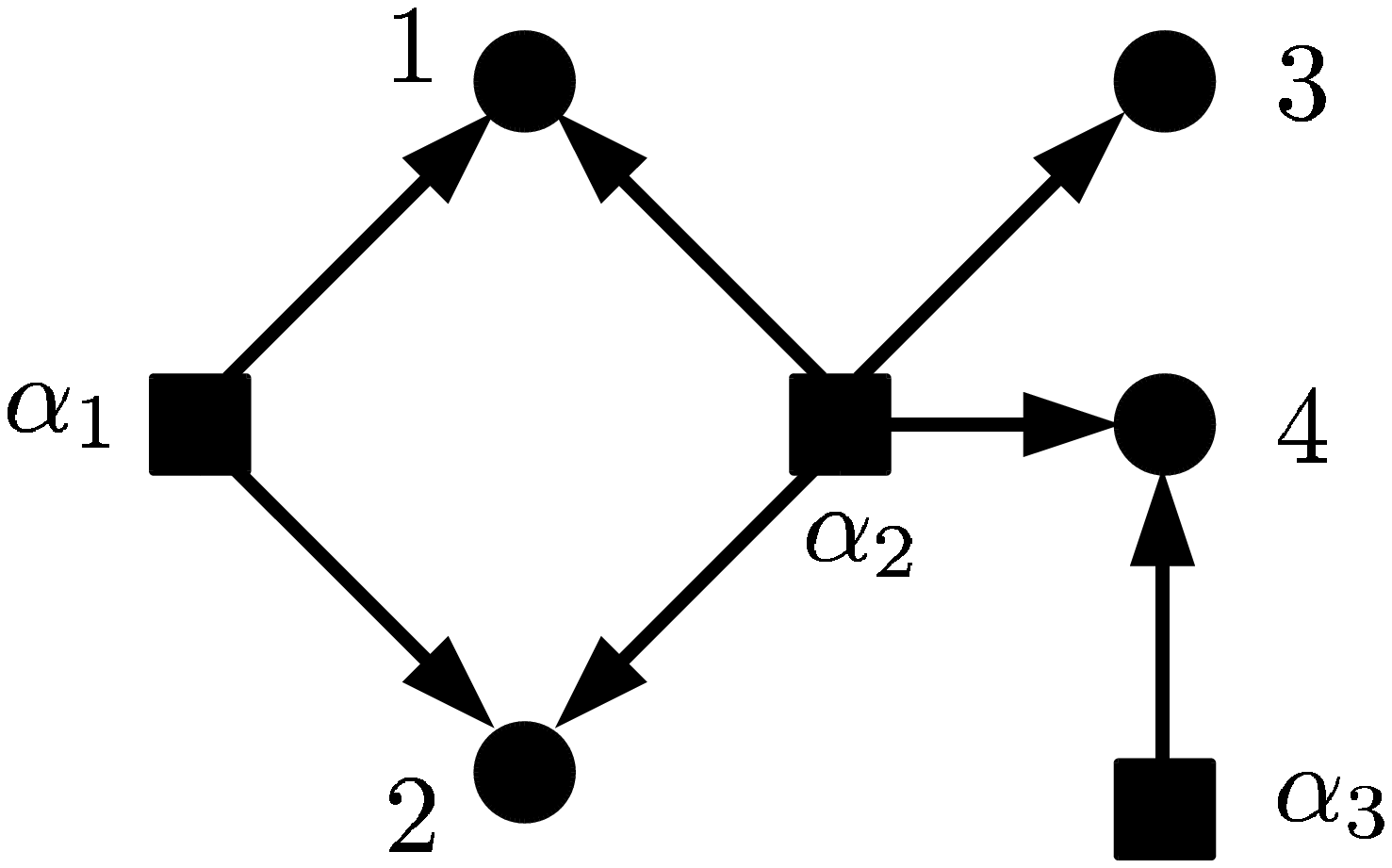}
\vspace{-2mm}
\caption{Directed graph representation. \label{fig:Dhypergraph}}
\end{center}
\end{minipage}
\begin{minipage}{.5\linewidth}
\begin{center}
\includegraphics[scale=0.3]{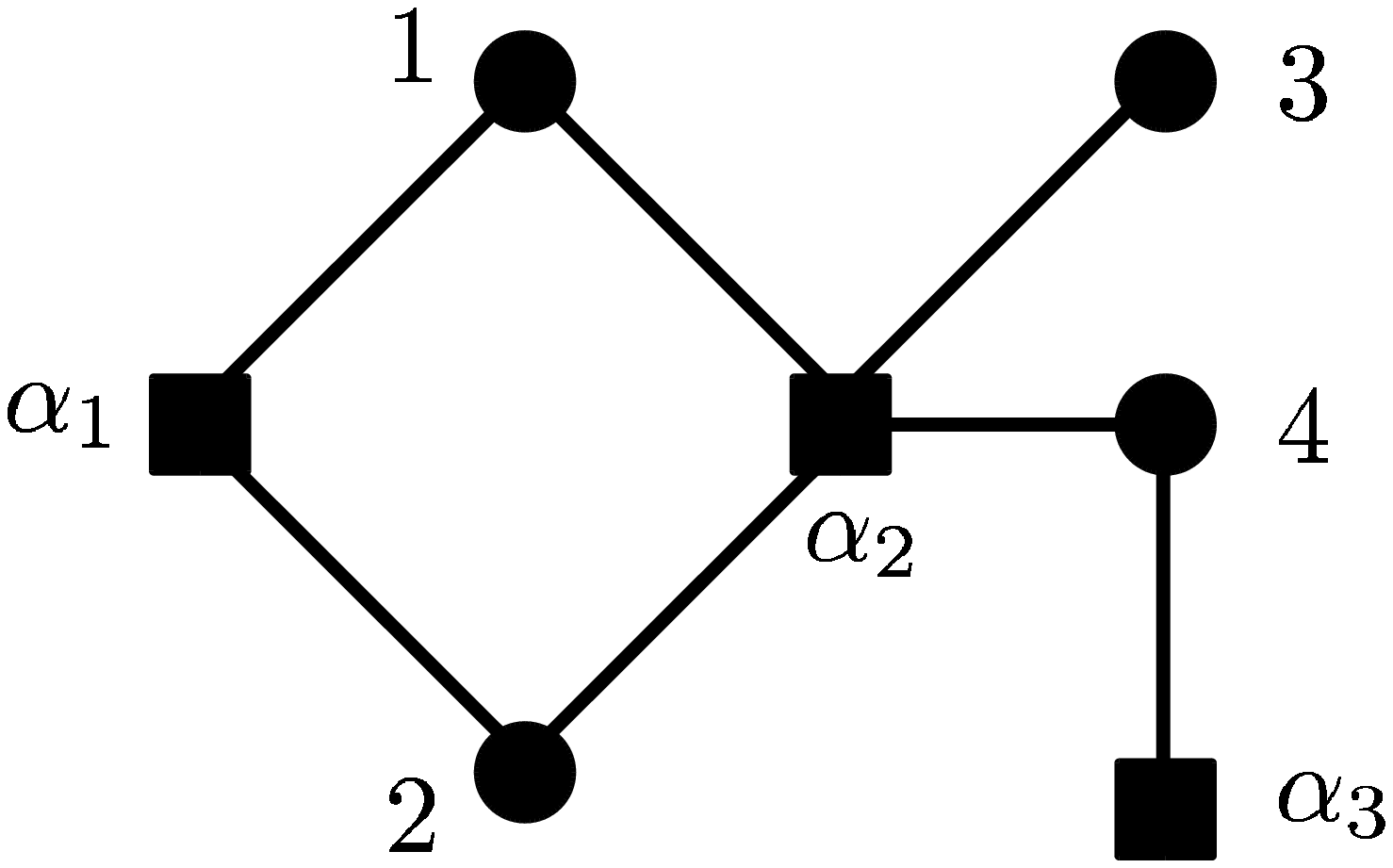}
\vspace{-1mm}
\caption{Bipartite graph representation. \label{fig:Bhypergraph}}
\end{center}
\end{minipage}
\end{figure}

Our primary interest is probability density functions
that have factorization structures represented by hypergraphs.
In such situations, a hypergraph is often referred to as a {\it factor graph} and a hyperedge as a {\it factor}.

\begin{defn}
Let $H=(V,F)$ be a hypergraph.
For each $i \in V$,
let $x_i$ be a variable that takes values in a set $\mathcal{X}_i$.
A probability density function $p$ on $x=(x_i)_{i \in V}$ is said to be {\it graphically factorized}
with respect to $H$ if it has the following factorized form
\begin{equation}
p(x)=\frac{1}{Z}
\prod_{\alpha \in F} \Psi_{\alpha}(x_{\alpha}), \label{defp}
\end{equation}
where $x_{\alpha}=(x_i)_{i \in \alpha}$, $Z$ is the normalization constant
and $\Psi_{\alpha}$ are positive
valued functions called {\it compatibility functions}.
A set of compatibility functions, giving a graphically factorized density function,
is called a {\it graphical model}.
The associated hypergraph $H$ is called the {\it factor graph} of the graphical model.
\end{defn}
Factor graphs are introduced by \citet{KFLfactor}.
Any probability density function on $\mathcal{X}= \prod_{i \in V} \mathcal{X}_i$ is trivially
graphically factorized with respect to the ``one-factor hypergraph'', where the unique factor includes all vertices.
It is more informative if the factorization involves factors of small size.
Our implicit assumption throughout this paper is that
for all factors $\alpha$, $\mathcal{X}_{\alpha}=\prod_i \mathcal{X}_i$ are small
enough, in the sense of cardinality or dimension, to be handled efficiently by computers.

\subsection{Loopy Belief Propagation algorithm}\label{sec:LBP}
Given a graphical model, our task is to solve inference problem such as
computation of marginal/conditional density functions and the partition function.
\BP (BP) efficiently computes the exact marginals of
a joint distribution that is factorized according to a tree-structured factor graph; %
\LBP (LBP) is a heuristic application of the algorithm for factor graphs with cycles,
showing successful performance in various problems.

First, in Subsection \ref{sec:expfamily}, we introduce a collection of exponential families
called {\it inference family} to formulate the LBP algorithm.
In order to perform inferences using LBP, we have to fix an \ifa that ``includes'' the given graphical model.
Our formulation is a variant of the approach by \citet{Wrepara}, where over-complete sufficient statistics are exploited.
The detail of the LBP algorithm is described in Subsections \ref{sec:basicLBP}.

\subsubsection{Exponential families and \Ifa}\label{sec:expfamily}
To clarify notations, here we summarize basic facts on exponential families.
Let ($\mathcal{X},\mathcal{B},\nu)$ be a measure space.
For given $n$ real valued functions ({\it sufficient statistics})
${\phi}(x)=(\phi_1(x),\ldots,\phi_n(x))$, an {\it exponential family} is given by
\begin{equation*}
 p(x;{\theta})=
\exp \left(
\sum_{i=1}^{N} \theta_i \phi_i(x)  - \psi({\theta})
\right), \qquad \quad
\psi({\theta}):= \log \int \exp \left( \sum_{i=1}^{N} \theta_i \phi_i(x) \right) {\rm d} \nu(x).
\end{equation*}
The {\it natural parameter}, ${\theta}$,
ranges over the set %
$\Theta:= {\rm int} \{{\theta} \in \mathbb{R}^{N};
\psi({\theta}) < \infty
\}$, where int denotes the interior of the set.
The function $\psi({\theta})$ is called the {\it log partition function}.
We always assume that the Hessian of this function (i.e. the covariance matrix) is invertible.
The derivative of the log partition function gives a bijective map
\begin{equation*}
 \Lambda : \Theta \ni \theta \longmapsto
\pd{\psi}{\theta}({\theta}) = \E{p_{\bstheta}}{\bsphi}
 \in Y := \Lambda(\Theta)
\end{equation*}
and this alternative parameter $\eta= \pd{\psi}{\theta}({\theta})$ is called
{\it expectation parameter}.
The inverse of this map is given by the derivative of the Legendre transform
$\varphi(\bseta) = \sup_{\bstheta \in \Theta} (  \sum_i \theta_i \eta_ i - \psi(\bstheta))
 =\E{p_{\Lambda^{-1} (\bseta)}}{\log p_{\Lambda^{-1}(\bseta)}}$.

\begin{ex}[Multinomial distributions]
\label{example:multinomial}
 Let $\mathcal{X}=\{0,\ldots,N-1\}$ be a finite set
with the uniform base measure.
One way of taking sufficient statistics is
\begin{equation}
 \phi_k(x)=
\begin{cases}
 1 \text{ \quad  if } x=k \\
 0 \text{ \quad  otherwise}
\end{cases}
\end{equation}
for $k=0,\ldots,N-2$.
Then the given exponential family is called {\it multinomial distributions} and
coincide with the all probability density functions on $\mathcal{X}$
that has positive probabilities for all elements of $\mathcal{X}$.
The region of \nparas is $\Theta=\mathbb{R}^{N-1}$ and
the of \eparas is the interior of the probability simplex.
That is, $ Y=\{ (y_1,\ldots,y_N); \sum_{k=1}^{N}y_k=1, y_k > 0 \}$.
\end{ex}

\begin{ex}[Gaussian distributions]
Let $\mathcal{X}=\mathbb{R}^{n}$ with the Lebesgue measure
and
The exponential family given by the sufficient statistics
$\phi(x)=(x_i,x_j x_k)_{1\leq i \leq n, 1 \leq j \leq k \leq n}$,
is called {\it Gaussian distributions}, consists of probability density functions of the form
\begin{equation}
 p(\bsx;\theta)= \exp \big(
\sum_{i \leq j}\theta_{ij}x_i x_j + \sum_{i}\theta_i x_i
-\psi (\theta)
\big). \nonumber
\end{equation}
\end{ex}

\begin{ex}[Fixed-mean Gaussian distributions]
For a given mean vector $\bsmu=(\mu_i)$,
the {\it fixed-mean Gaussian distributions} is the exponential family obtained by the sufficient
statistics $\phi(x)=\{(x_i-\mu_i)(x_j-\mu_j)\}_{1 \leq i \leq j \leq n}$.

\end{ex}

Here and below, we construct a set of exponential families.
In order to perform inferences using LBP for
a given graphical model, we have to fix a ``family'' that includes the given probability density function.

Let $H=(V,F)$ be a hypergraph.
First, for each vertex $i$, we consider an exponential family $\mathcal{E}_i$ with a sufficient statistic $\phi_i$
and a base measure $\nu_i$ on $\mathcal{X}_i$.
A \npara, \epara, the log partition function and its Legendre transform are denoted by $\theta_i$, $\eta_i$,
$\psi_i$ and $\varphi_i$ respectively.
Secondly, for each factor $\alpha=\fai$,
we give an exponential family $\mathcal{E}_{\alpha}$ on $\mathcal{X}_{\alpha}= \prod_{i \in \alpha} \mathcal{X}_i$
with the base measure $\nu_{\alpha}=\prod_{i \in \alpha} \nu_i$ and a sufficient statistic $\fa{\phi}$ of the form
\begin{equation}
 \fa{\phi}(x_{\alpha})=(\pa{\phi}(x_{\alpha}),\phi_{i_1}(x_{i_1}),\ldots,\phi_{i_{d_{\alpha}}}(x_{i_{d_{\alpha}}}) ).
\end{equation}
An important point is that $\fa{\phi}$ includes the sufficient statistics for $i \in \alpha$ as its components in addition to $\pa{\phi}$ indexed by $\alpha\in F$.
The \npara, \epara, log partition function and its Legendre transform are denoted by
\begin{equation}
 \fa{\theta}=(\pa{\theta},\va{\theta}{i_1},\ldots,\va{\theta}{ i_{d_{\alpha}}} ) \in \fa{\Theta}, \quad
 \fa{\eta}=(\pa{\eta},\va{\eta}{i_1},\ldots,\va{\eta}{i_{d_{\alpha}}}) \in \fa{Y}, \quad
 \psi_{\alpha} \text{ and } \varphi_{\alpha}.
\end{equation}

The following assumption is indispensable to our analysis:
\begin{asm}\label{asm:invetible}
For all $i \in V$ and $\alpha \in F$,
we assume that the Hessian of the log partition functions
, $\psi_i$ and $\psi_{\alpha}$,
(i.e. the covariance matrix) are invertible in the parameter spaces.
\end{asm}

In order to use these exponential families $\mathcal{E}_{\alpha}$ and $\mathcal{E}_i$ for LBP, we need another assumption:
the family is ``closed'' under marginalization operation.
This type of condition on exponential families is also considered in other litterateurs \citep{Mardia2009}.
\begin{asm}
[Marginally closed assumption] \label{asm:marginallyclosed}
For all pair of $i \in \alpha$,
\begin{equation}
 \int  p(x_{\alpha}) {\rm d} \nu_{\alpha \smallsetminus i} ( x_{\alpha \smallsetminus i})
\in \mathcal{E}_i \quad \text{ for all } p \in \mathcal{E}_{\alpha}.
\end{equation}
\end{asm}

\begin{defn}
\label{defn:ifa}
If a collection of the exponential families $\mathcal{I}:=\{\mathcal{E}_{\alpha}, \mathcal{E}_i \}$
given by sufficient statistics $( \pa{\phi}(x_{\alpha}), \phi_i (x_i) )_{ \alpha \in F,i \in V}$ as above 
satisfies Assumptions \ref{asm:invetible} and \ref{asm:marginallyclosed}, it
is called an {\it \ifa} associated with a hypergraph $H$.
An \ifa is called {\it pairwise} if the associated hypergraph is a graph.
\end{defn}
An \ifa has a parameter set $\Theta=\prod_{\alpha} \fa{\Theta} \times \prod_i \Theta_i$, which is bijectively mapped
to the dual parameter set $Y=\prod_{\alpha} \fa{Y} \times \prod_i Y_i$ by the maps of respective components.
An \ifa naturally defines an exponential family on $\mathcal{X}=\prod_i \mathcal{X}_i$
of the sufficient statistic $( \pa{\phi}(x_{\alpha}), \phi_i (x_i) )_{ \alpha \in F,i \in V}$.
We denote it by $\mathcal{E}(\mathcal{I})$.

\begin{ex}[Binary pairwise inference family]
Consider the case that a graph $G=(V,E)$ is the factor graph.
For each $i \in V$, we define an exponential family $\mathcal{E}_i$ on $\mathcal{X}_i=\{0,1\}$
defined by $\phi_i(x_i)=x_i$.
For each $\{i,j\} \in E$, we also define multinomial exponential family $\mathcal{E}_{\{i,j\}}$ on $\{0,1\}^2$
by $\phi_{\{i,j\}}(x_i,x_j)=(x_i,x_j,x_i x_j)$, where $\pij{\phi}(x_i)=x_i x_j$.
Then these exponential family gives an \ifa
since Assumption \ref{asm:marginallyclosed} is trivially satisfied.
\end{ex}

\begin{ex}[Multinomial inference family]
Let $\mathcal{E}_i$ be an exponential family of multinomial distributions.
Choosing functions $\pa{\phi}(x_{\alpha})$, 
we can make the $\mathcal{E}_{\alpha}$
being multinomial distributions on $\mathcal{X}_{\alpha}$;
more precisely, we choose $\pa{\phi}(x_{\alpha})$ so that 
the components of $\phi_{\alpha}(x_\alpha)$, which are regarded as $\prod_{i} |\mathcal{X}_i|$ dimensional vectors,
are linearly independent.
Then we obtain an \ifa called a {\it multinomial inference family}.
\end{ex}

\begin{ex}[Gaussian  inference family]
We consider the case\footnote{Extensions to high dimensional case, i.e. $\mathcal{X}_i=\mathbb{R}^{r_i}$, is straight forward.}
that $\mathcal{X}_i = \mathbb{R}$ .
For Gaussian case, given a factor graph $H=(V,F)$,
the sufficient statistics are given by
\begin{equation}
 \phi_{i}(x_i)=(x_i,x_i^2), \qquad
 \pa{\phi}(x_{\alpha})=(x_i x_j)_{i , j \in \alpha, i \neq j}. \nonumber
\end{equation}
Then the \ifa is called {\it Gaussian \ifa}.
Assumption \ref{asm:marginallyclosed} is satisfied because a marginal of a Gaussian density function is a Gaussian density function.
Fixed-mean \ifa is analogously defined by $\phi_{i}(x_i)=(x_i - \mu_i )^2$ and
$\pa{\phi}(x_{\alpha})=((x_i -\mu_i)(x_j - \mu_j))_{i , j \in \alpha, i \neq j}$.
Usually, for Gaussian cases, the factor graph $H$ is a graph rather than hypergraphs;
thus, we only consider Gaussian \ifas on graphs.
\footnote{Extensions to the cases of hypergraphs are also straightforward.}
\end{ex}

\subsubsection{LBP algorithm}\label{sec:basicLBP}
The LBP algorithm calculates the approximate marginals of a given graphical model $\Psi = \{ \Psi_{\alpha}\}$
using the inference family \ifa $\mathcal{I}$.
We always assume that the \ifa includes the given probability density function:
\begin{asm}
\label{asm:modelindludes}
For every factor $\alpha \in F$, there exists $\fa{\bar{\theta}}$ s.t.
\begin{equation}
 \Psi_{\alpha}(x_{\alpha}) = \exp \left( \inp{\fa{\bar{\theta}} }{ \fa{\phi} (x_{\alpha})}  \right). \label{eq:asm:modelindludes}
\end{equation}
\end{asm}
This is equivalent to the assumption
\begin{equation}
p(x)=  \frac{1}{Z} \prod_{\alpha} \Psi_{\alpha}(x_{\alpha}) \in \mathcal{E}(\mathcal{I})
\end{equation}
up to trivial re-scaling of $\Psi_{\alpha}$, which does not affect LBP algorithm.

The procedures of the LBP algorithm is as follows \citep{KFLfactor}.
For each pair of a vertex $i \in V$ and a factor $\alpha \in F$ satisfying $i \in \alpha$,
an initialized message is given in the form of
\begin{equation}
 m_{\edai }^{0}(x_i) = \exp ( \inp{\mu_{\edai}^{0}}{\phi_i(x_i)} ), \label{messageform}
\end{equation}
where the choice of $\mu_{\edai}^{0}$ is arbitrary.
The set $\{  m_{\edai }^{0} \}$ or $\{ \mu_{\edai}^{0}\}$ is called an {\it initialization} of the LBP algorithm.
At each time $t$, the messages are updated by the following rule:
\begin{equation}
m^{t+1}_{\edai}(x_i)
=\omega
\int
\Psi_{\alpha}(x_{\alpha})
\hspace{-1mm}
\prod_{j \in \alpha, j \neq i}
\prod_{\beta \ni j, \beta \neq \alpha}
\hspace{-1mm} m^{t}_{\edbj}(x_j)
\hspace{1mm} {\rm d}\nu_{\alpha \smallsetminus i}({x_{\alpha \smallsetminus i}})
\qquad (t \geq 0), \label{LBPupdate}
\end{equation}
where $\omega$ is a certain scaling constant.\footnote{
Here and below, we do not care about the integrability problem.
For multinomial and Gaussian cases, there are no problems.}
See Fig~\ref{fig:LBPupdate} for the illustration of this message update scheme.
From Assumptions \ref{asm:marginallyclosed} and \ref{asm:modelindludes},
the messages keep the form of Eq.~(\ref{messageform}).
\begin{figure}
\begin{center}
\includegraphics[scale=0.25]{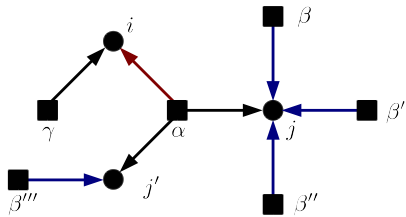}
\vspace{-2mm}
\caption{The blue messages contribute to the red message at the next time step. \label{fig:LBPupdate}}
\end{center}
\end{figure}

Since this update rule simultaneously generates all messages of time $t+1$ by those of
time $t$, it is called a {\it parallel update}.
Another possibility of the update is a {\it sequential update}, where, at each time step,
one message is chosen according to some prescribed or random order of directed edges.
In this paper, we mainly discuss the parallel update.

We repeat the update Eq.~(\ref{LBPupdate}) until the messages converge to a fixed point,
though this procedure is not guaranteed to converge.
Indeed, it sometimes exhibits oscillatory behaviors. %
The set of LBP fixed points does not depend on the choices of the update rule,
but converging behavior, or {\it dynamics}, does depend on the choices.

If the algorithm converges, we obtain the fixed point messages $\{m^{*}_{\edai}\}$
and {\it beliefs} %
that are defined by
\begin{align}
&b_{i}(x_i):= \omega
\prod_{\alpha \ni i} m_{\edai}^{*}(x_i) \label{eq:defbelief1}\\
&b_{\alpha}(x_{\alpha})
:=\omega
\Psi_{\alpha}(x_{\alpha})
\prod_{j \in \alpha }
\prod_{\beta \ni j, \beta \neq \alpha} m^{*}_{\edbj}(x_j),  \label{eq:defbelief2}
\end{align}
where $\omega$ denotes (not necessarily the same) normalization constants that require
\begin{equation}
 \int b_i(x_i) {\rm d} \nu_i =1 \quad \text{ and } \quad \int b_{\alpha}(x_{\alpha}) {\rm d} \nu_{\alpha}=1.
\end{equation}
Note that beliefs automatically satisfy the conditions $b_{\alpha}(x_{\alpha}) > 0$ 
and
\begin{equation}
\int b_{\alpha}(x_{\alpha}) {\rm d} \nu_{\alpha \smallsetminus i}({x_{\alpha \smallsetminus i}}) =b_i(x_i). \label{eq:localconsistency}
\end{equation}
The beliefs are used for approximation of the true marginal density functions.

If $H$ is a tree,
the LBP algorithm stops at most $|\vec{E}|$ updates and
the computed beliefs are equal to the exact marginals of the given density function.

\subsection{Bethe free energy and characterization of LBP fixed points}\label{sec:BFE}
The Bethe approximation was initiated by \citet{Bethe}
and was found to be essentially equivalent to LBP by \citet{YFWGBP}.
The modern formulation for presenting the approximation is a variational problem
of the {\it Bethe free energy} \citep{Anote}.
In this subsection, we summarize these facts in our settings.

First, we should introduce the Gibbs free energy function
because the Bethe free energy function is a computationally tractable approximation of the Gibbs free energy function.
For given graphical model $\Psi=\{ \Psi_{\alpha}\}$, the {\it Gibbs free energy} $F_{Gibbs}$ is
a convex function over the set of probability distributions $\hat{p}$ on $x=(x_i)_{i \in V}$ defined by
\begin{equation}
 F_{Gibbs}(\hat{p})= \int \hat{p}(x) \log \left(
\frac{\hat{p}(x)}{\prod_{\alpha} \Psi_{\alpha}(x_{\alpha})}
\right) {\rm d}\nu(x), \label{def:GibbsFE}
\end{equation}
where $\nu = \prod_{i \in V} \nu_i$ is the base measure on $\mathcal{X}= \prod_{i \in V} \mathcal{X}_i$.
Using Kullback-Leibler divergence $D(q||p)=\int\hat{p}\log(q/p)$,
Eq.~(\ref{def:GibbsFE}) comes to
$
 F_{Gibbs}(\hat{p})= D(\hat{p}||p) - \log Z.
$
Therefore, the exact density function Eq.~(\ref{defp})
is characterized by a variational problem
\begin{align}
 p(x)= \argmin_{\hat{p}} F_{Gibbs}( \hat{p} ), \label{eq:Gibbsvariation}
\end{align}
where the minimum is taken over
all probability distributions on $x$.
As suggested from the name of ``free energy'',
the minimum value of this function is equal to $- \log Z$.

\label{sec:twoBfes}
In many cases including discrete variables, computing values of the Gibbs free energy function is intractable in general
because the integral in Eq.~(\ref{def:GibbsFE}) is indeed a sum over $|\mathcal{X}|= \prod_i |\mathcal{X}_i|$ states.
We introduce functions called \Bfe that does not include such an exponential number of state sum.

\begin{defn}
The \Bfe (BFE) function is a function of \eparas.
For a given \ifa $\mathcal{I}$,
define
$L(\mathcal{I}): = \{ \bs{\eta}=\{\fa{\eta},\eta_{i}\} \in Y | \va{\eta}{i}=\eta_i \ {}^{\forall}( i \in \alpha) \}$
\footnote{We often write $L(\mathcal{I})$ as $L$ when $\mathcal{I}$ is obvious from the context.
Since $Y=\prod_{\alpha} \fa{Y} \times \prod_i Y_i$ is convex, $L$ is a convex set.
If the \ifa is multinomial, the closure of this set is called {\it local polytope} \citep{WJgraphical,WJvariational}.}.
On this set, the \Bfe function is defined by
\begin{equation}
 F(\bs{\eta}):=
-\sum_{\alpha \in F} \inp{\fa{\bar{\theta}}}{ \fa{\eta} } +
\sum_{\alpha \in F}\varphi_{\alpha}(\fa{\eta} ) +
\sum_{i \in V} (1-d_i)\varphi_i(\eta_i), \label{defn:Bfe}
\end{equation}
where $\fa{\bar{\theta}}$ is the \npara of $\Psi_{\alpha}$ in Eq.~(\ref{eq:asm:modelindludes}).
\end{defn}
An expectation parameter specifies a probability density function in the exponential family.
Thus, $\bseta \in Y$ specifies $\beliefs$, where
$b_{\alpha}(x_{\alpha}) \in \mathcal{E}_{\alpha}$ and $b_{i}(x_{i}) \in \mathcal{E}_{i}$.
The constraint $\va{\eta}{i}=\eta_i$ means that
\begin{equation*}
 \int \phi_i(x_i) b_{\alpha}(x_{\alpha}) {\rm d} \nu_{\alpha} = \int \phi_i(x_i) b_i(x_i)  \nu_{i}.
\end{equation*}
Under Assumption \ref{asm:modelindludes}, this condition is equivalent to
$\int b_{\alpha}(x_{\alpha}) {\rm d} \nu_{\alpha \smallsetminus i} = b_i(x_i)$
because a probability density function in $\mathcal{E}_i$ is specified by the expectation of $\phi_i(x_i)$.
An element of $L$ is called a set of {\it pseudomarginals}. %
Therefore, we have the following identification
\begin{equation*}
 L = \Big\{ \beliefs | ~b_{\alpha}(x_{\alpha}) \in \mathcal{E}_{\alpha}, ~b_{i}(x_{i}) \in \mathcal{E}_{i}
 \quad and \int b_{\alpha}(x_{\alpha}) {\rm d} \nu_{\alpha \smallsetminus i} = b_i(x_i) \Big\}.
\end{equation*}
The second condition is called {\it local consistency}.
Under this identification, the \Bfe function is
\begin{align}
 F(\beliefsw)= -\sum_{\alpha \in {F}} \int
 b_{\alpha}(x_{\alpha})\log\Psi_{\alpha}(x_{\alpha})   {\rm d}\nu_{\alpha}
& + \sum_{\alpha \in {F}} \int  b_{\alpha}(x_{\alpha})\log b_{\alpha}(x_{\alpha}) {\rm d} \nu_{\alpha} \nonumber \\
&+ \sum_{i \in V}(1-d_i) \int b_i(x_i)\log b_i(x_i) {\rm d}\nu_i. \nonumber
\end{align}

If $H$ is a tree, the variational problem of the Bethe free energy over $L$
is equivalent to that of the Gibbs free energy in the following sense.
See \citet{WJgraphical} for more details.
First, it can be shown that, for any $\beliefsw \in L$,
\begin{equation}
 \Pi(\beliefsw) :=\prod_{\alpha}b_{\alpha}(x_{\alpha}) \prod_{i} b_i(x_i)^{1-d_i}   \label{factorb}
\end{equation}
is a probability density function because it is summed up to one.
For these type of density functions, we can see that the Gibbs free energy function is equal to
the \Bfe function: $F = F_{Gibbs}\circ \Pi$.
Secondly, it is also known that the true density function $p$ for a tree has the factorization of the form Eq.~(\ref{factorb}).
Therefore, the variational problem Eq.~(\ref{eq:Gibbsvariation}) reduces to that of the \Bfe function over $L$.

For general factor graphs,
the Bethe variational problem approximates the Gibbs variational problem and
a minimizer of the Bethe problem can be used to approximate the marginal density function.
As shown by \cite{PAstat}, the \Bfe function is convex if the factor graph has at most one cycle.
Therefore, the minimization of the \Bfe is easy for these cases.
In general, however, the convexity of $F$ is broken as the nullity of the underlying factor graph becomes large,
yielding multiple minima.
Though the functions $\varphi_{\alpha}$ and $\varphi_i$ are convex,
the negative coefficients $(1-d_i)$ makes the function $F$ complex.
The positive-definiteness of the Hessian of the Bethe free energy will be analyzed in
Section~\ref{sec:key} and \ref{sec:pdconv}.

\label{sec:LBPcharacterizations}
The \Bfe function gives an alternative description of the LBP fixed points.
The following fact is shown by \citet{YFWGBP}; LBP finds a stationary point of the \Bfe function,
which is a necessary condition of the minimality.
We give the proof in our term in Appendix~\ref{app:DetailProofs}.

\begin{thm}
 \label{thm:LBPcharacterizations}
Let $\mathcal{I}$ be an \ifa and $\Psi=\{\Psi_{\alpha}\}$ be a graphical model.
The following sets are naturally identified each other.
\begin{enumerate}
 \item The set of fixed points of \lbp.
 \item The set of stationary points of $F$ over $L(\mathcal{I})$.
\end{enumerate}
\end{thm}

\section{Graph zeta function}

The aim of this section is to introduce the graph zeta function and develop some results, which are used in the later sections.

{\it Ihara's graph zeta function}
was originally introduced by Y.~\cite{Idiscrete} for a certain algebraic object, 
and was abstracted and extended to be defined on arbitrary finite graphs
by J.~P.~\cite{Strees}, \cite{SL-functions} and \cite{Bass}.
{\it The edge zeta function} is a multi-variable generalization of Ihara's graph zeta function,
allowing arbitrary scalar weight for each directed edge \citep{STzeta1}.
Extending those graph zeta functions,
we introduce a graph zeta function defined on hypergraphs with matrix weights.

The central result of this section is the \IB type determinant formula in Subsection~\ref{sec:detIhara}.
This formula plays an important role in deriving the positive definiteness condition in Subsection~\ref{sec:zetapositive}.
These results are utilized to establish the relations between this zeta function and the LBP algorithm in the next section.

\subsection{Definition of the graph zeta function}
\label{sec:defgraphzeta}
In the first part of this subsection,
we further introduce basic definitions and notations
of hypergraphs required for the definition of our graph zeta function.

Let $H=(V,F)$ be a hypergraph.
As noted before,
it can be regarded as a directed graph $\fgdefn$.
For each edge $e=(\edai) \in \vec{E}$,
$s(e)=\alpha \in F$ is the
{\it starting hyperedge} of $e$ and
$t(e)=i \in V$ is the {\it terminus vertex} of $e$.
If two edges $e,e'\in \vec{E}$ satisfy conditions
$t(e) \in s(e')$ and $t(e) \neq t(e')$, this pair is denoted by $\ete{e}{e'}$.
(See Figure \ref{fig:edgerelation}.)
A sequence of directed edges $(e_1,\ldots,e_k)$ is said to be a
{\it closed geodesic} if $\ete{e_l}{e_{l+1}}$
for $l \in \mathbb{Z}/k\mathbb{Z}$.
For a closed geodesic $c$, we may form the {\it m-multiple} $c^{m}$ by repeating $c$ $m$-times.
If $c$ is not a multiple of strictly shorter closed geodesic, $c$ is said to be {\it prime}.
For example, a closed geodesic $c=(e_1,e_2,e_3,e_1,e_2,e_3)$ is not prime
because $c=(e_1,e_2,e_3)^2$.
A closed geodesic $c=(e_1,e_2,e_3,e_4,e_1,e_2,e_3)$ is prime because it is not $c \neq c'^m$
for any $c'$ and $m (\geq 2)$.
Two closed geodesics are said to be {\it equivalent} if one is obtained by cyclic
permutation of the other.
For example, closed geodesics $(e_1,e_2,e_3), (e_2,e_3,e_1) \text{ and } (e_3,e_1,e_2)$ are equivalent.
An equivalence class of a prime closed geodesic is called a {\it prime cycle}.
The set of prime cycles of $H$ is denoted by $\mathfrak{P}_H$.

\begin{figure}
\begin{center}
\includegraphics[scale=0.25]{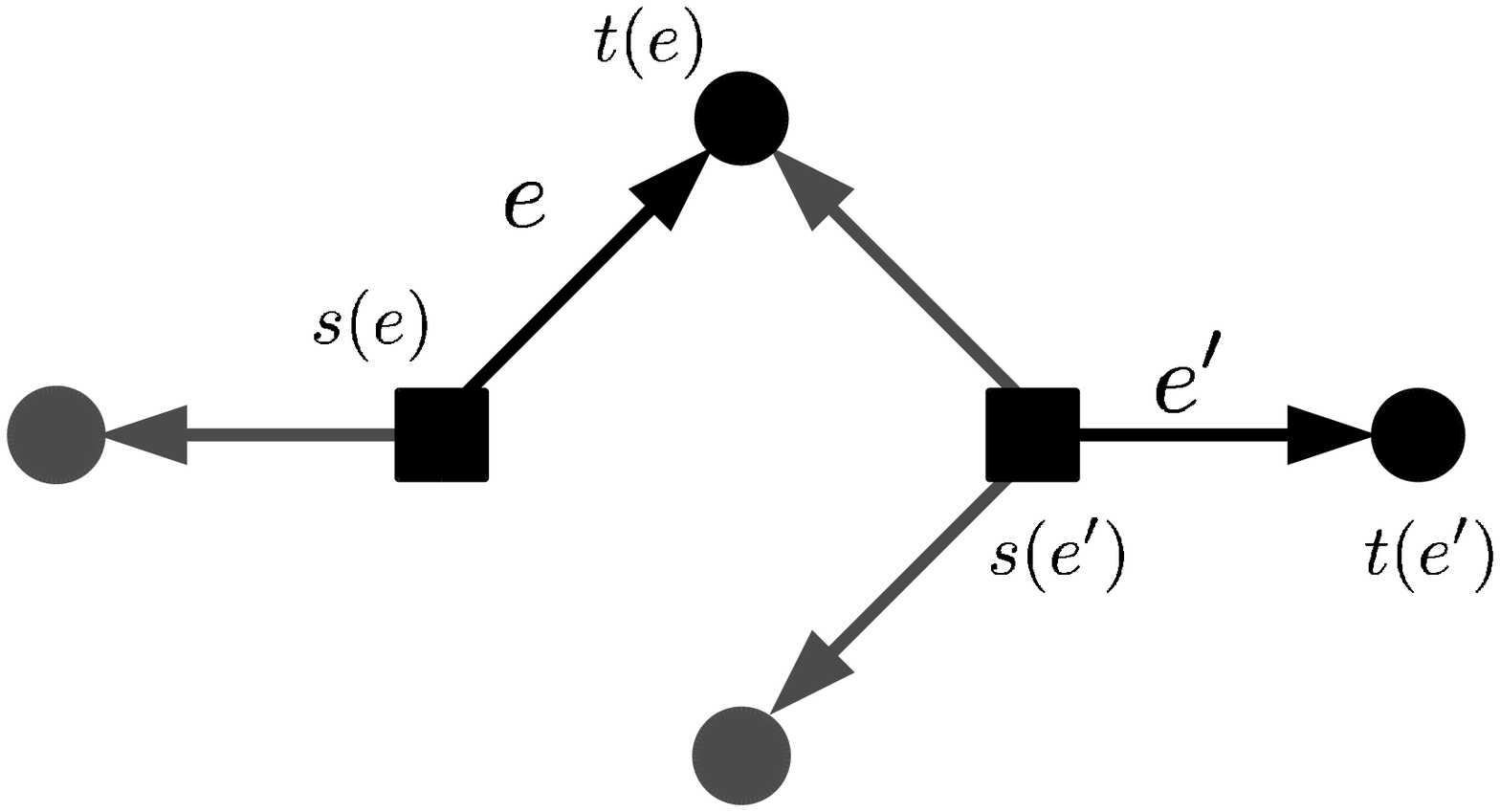}
\vspace{-1mm}
\caption{Example of the relation $\ete{e}{e'}$. \label{fig:edgerelation}}
\end{center}
\end{figure}

If $H$ is a graph (i.e. $d_{\alpha}=2$ for all $\alpha \in F$),
these definitions reduce to standard definitions \citep{KSzeta}.
(We will explicitly give them in Subsection \ref{sec:specialIB}.)
In this case, a factor $\alpha=\{i,j\}$ is identified with an undirected edge $ij$ and
$(\edai)$ is identified with a directed edge $(\edij)$.

Usually, in graph theory, Ihara's graph zeta function is a uni-variate function and associated with a graph.
Our graph zeta function is much more involved:
it is defined on a hypergraph having weights of matrices.
To define matrix weights, we have to prescribe its sizes;
we associate a positive integer $r_e$ with each edge $e \in \vec{E}$.

Here are additional notations used in the following definition.
The set of functions
\footnote{In mathematical usage, this is not a ``function'' because it takes a value on a different set for each argument $e \in \vec{E}$.
However, we do not stick this point.} 
on $\vec{E}$ that take values on $\mathbb{C}^{r_e}$
for each $e \in \vec{E}$ is denoted by $\vfe$.
The set of $n_1 \times n_2$ complex matrices is denoted by $\mat{n_1}{n_2}$.

\begin{defn}
Assume that for each $\etea$, a matrix weight $u_{\etea} \in \mat{r_e}{r_{e'}}$ is associated.
For this matrix weights $\bs{u}=\{u_{\etea}\}$,
the graph zeta function of $H$ is defined by
\begin{equation}
 \zeta_{H}(\bs{u}):=
\prod_{\mathfrak{p} \in \mathfrak{P}_H}
\frac{1}{\det \big( I- \pi(\mathfrak{p}) \big) },  \nonumber
\end{equation}
where $\pi(\mathfrak{p}):=$
$u_{e_k \rightharpoonup e_1}\ldots u_{e_2 \rightharpoonup e_3}  u_{e_1 \rightharpoonup e_2}$
for $\mathfrak{p}=(e_1,\ldots,e_k)$.
\end{defn}
Since $\det(I_n-AB)=\det(I_m-BA)$ for $n \times m $ and $m \times n$ matrices $A$ and $B$,
$\det( I- \pi(\mathfrak{p}))$ is well defined for an equivalence class $\mathfrak{p}$.
The definition is an analogue of the Euler product formula of the Riemann zeta function which is represented by the
product over all the prime numbers.

If $H$ is a graph and $r_{e}=1$ for all $e \in \vec{E}$,
this zeta function reduces to the edge zeta function by \cite{STzeta1}.
If in addition all these scalar weights are set to be equal, i.e. $u_{\etea}=u$,
the zeta function reduces to the Ihara zeta function.
These reductions will be discussed in Subsection \ref{sec:specialIB}.
Moreover, for general hypergraphs,
we obtain the one-variable hypergraph zeta function
by setting all matrix weights to be the same scalar $u$ \citep{Shypergraph}.

\begin{ex} \label{example1}
$\zeta_{H}(\boldsymbol{u})^{}=1$ if $H$ is a tree.
For 1-cycle graph $C_N$ of length
$N$, the prime cycles are $(e_1,e_2,\ldots,e_N)$ and
$(\bar{e}_N,\bar{e}_{N-1},\ldots,\bar{e}_1)$. (See Figure \ref{fig:prime1cycle}.)
The zeta function is
\begin{small}
\begin{equation*}
 \zeta_{C_N}(\boldsymbol{u})=
\det(I_{r_{e_1}}- u_{e_N \rightharpoonup e_1}\ldots u_{e_2 \rightharpoonup e_3}  u_{e_1 \rightharpoonup e_2}        )^{-1}
\det(I_{r_{\bar{e}_N}}- u_{\bar{e}_1 \rightharpoonup \bar{e}_N}\ldots u_{\bar{e}_{N-1} \rightharpoonup \bar{e}_{N-2}}  u_{\bar{e}_N \rightharpoonup \bar{e}_{N-1}}  )^{-1}.
\end{equation*}
\end{small}
\end{ex}

\begin{figure}
\begin{center}
\includegraphics[scale=0.25]{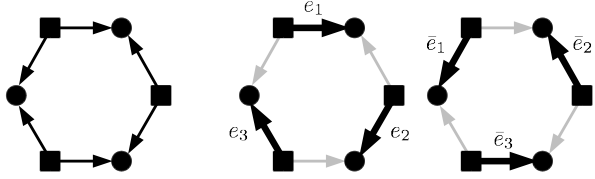}
\vspace{-1mm}
\caption{$C_3$ and its prime cycles. \label{fig:prime1cycle}}
\end{center}
\end{figure}

Except for the above two types of hypergraphs,
the number of prime cycles is infinite.
Therefore, rigorously speaking, we have to care about the convergence of the product
and restrict the definition for sufficiently small matrix weights $\bsu$.
However, as we will see below, the zeta function
has a determinant formula and is well defined on the whole space of matrix weights.
The proof is given in Appendix~\ref{app:DetailProofs}.

\begin{thm}
[The first determinant formula of zeta function]
\label{thm:det1}
We define a linear operator $\matmu : \vfe \rightarrow \vfe$ by
\begin{equation}
 \matmu f (e) =  \sum_{e':\etea} u_{\etea}f(e') \qquad f \in \vfe. \nonumber
\end{equation}
Then, the following formula holds
\begin{equation}
  \zeta_{G}(\bs{u})^{-1}=
\det(I- \matmu). \nonumber
\end{equation}
\end{thm}
This type of determinant formula is well known in the context of graph zeta functions;
in fact this theorem is a straightforward generalization of Theorem 3 of \citet{STzeta1}.
In the next section we derive a new determinant formula of the zeta function by manipulating
the matrix $\matmu$ in the above determinant.

Note that the matrix representation of the operator $\matmu$ is
\begin{equation}
\matmu_{e,e'}=
\begin{cases}
u_{\etea}  \qquad \text{if } \etea \\
0          \qquad \qquad \text{otherwise.}
\end{cases} \nonumber
\end{equation}
The simplification of this matrix obtained by setting $r_e=1$ and $u=1$
is called {\it directed edge matrix} and denoted by $\mathcal{M}$ \citep{STzeta1}.
 \cite{KSzeta} call this matrix a {\it Perron-Frobenius operator}.
A noteworthy difference, in our and their definitions, is that directions of edges are opposite,
because we choose the directions to be consistent with illustrations of the LBP algorithm.

\subsection{Determinant formula of \IB type}\label{sec:detIhara}
In the previous subsection, we have shown that the zeta function is expressed as a determinant
of size $\sum_{e \in \vec{E} } r_{e}$.
In this subsection, we show another determinant expression with additional assumptions on the matrix weights.
The formula is called {\it \IB type determinant formula} and plays a key role in proofs of Theorem~\ref{thm:zetapositive}
and Theorem~\ref{thm:BZ}.

In the rest of this subsection, we fix a set of positive integers $\{r_i\}_{i \in V}$
associated with vertices.
Let $\{u^{\alpha}_{\edij}\}_{\alpha \in F, i,j \in \alpha}$ be a set of matrices 
$u^{\alpha}_{\edij} \in \mat{r_j,r_i}$.
Our additional assumption on the set of matrix weights, which is the argument of the zeta function, is that
\begin{equation}
 r_e:=r_{t(e)} \text{~ and  ~} u_{\etea}:=u^{s(e)}_{t(e') \rightarrow t(e)}. \nonumber
\end{equation}
Then the graph zeta function can be seen as a function of $\bs{u}=\{u^{\alpha}_{\edij}\}$.
With slight abuse of notation, it is also denoted by $\zeta_{H}(\bsu)$.
Later in Section~\ref{sec:key},
$r_i$ corresponds to the dimension of the sufficient statistic $\phi_{i}$, and
$u^{\alpha}_{\edij}$ to a matrix $\var{b_j}{\phi_j}^{-1} \cov{b_{\alpha}}{\phi_j}{\phi_i}$.

To state the \IB type determinant formula,
we introduce a linear operator $\bs{\iota}(\bs{u}): \vfe \rightarrow \vfe$
defined by
\begin{equation}
 (\bs{\iota}(\bs{u})f)(e):=
\sum_{e': {s(e')=s(e) \atop t(e')\neq t(e) }} u^{s(e)}_{\ed{t(e')}{t(e)}}f(e')  \qquad f \in \vfe.  \nonumber %
\end{equation}
\begin{figure}
\begin{center}
\includegraphics[scale=0.25]{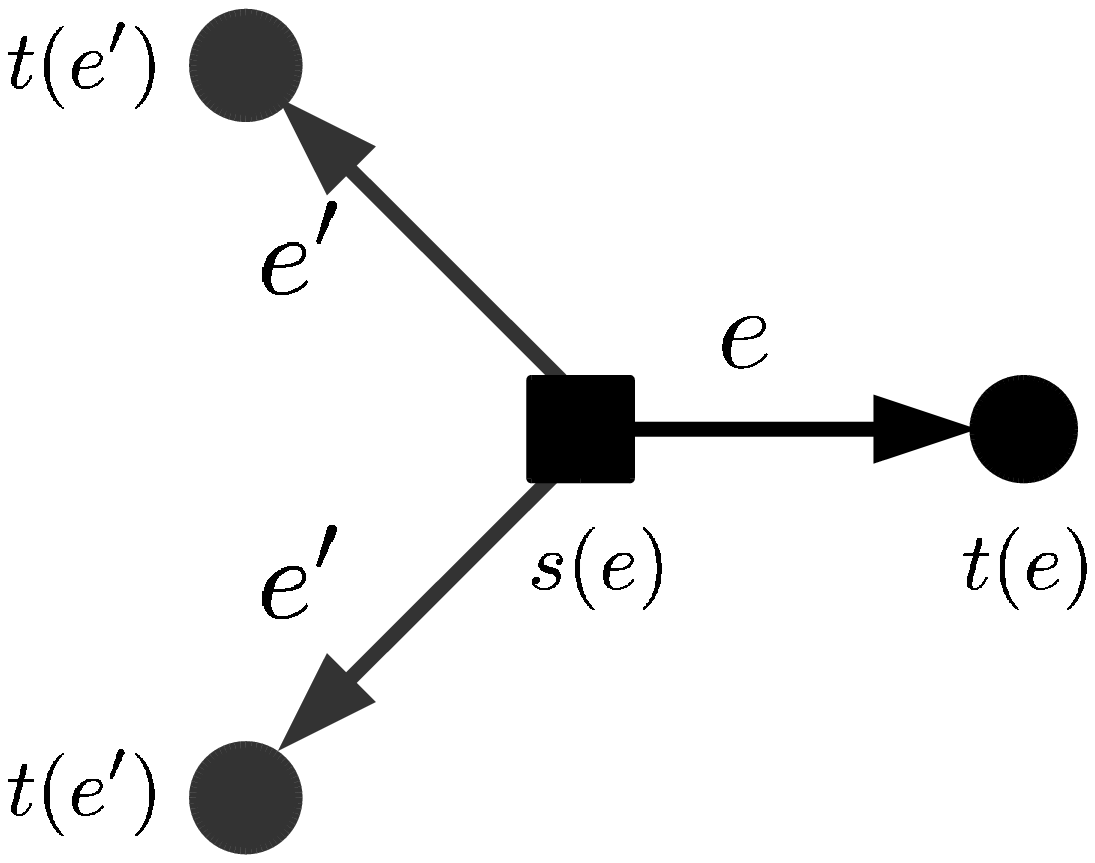}
\vspace{0mm}
\caption{Illustration for the definition of $\bs{\iota}(\bsu)$. \label{fig:iotaillustration}}
\end{center}
\vspace{0mm}
\end{figure}
The matrix representation of $\bs{\iota}(\bs{u})$ is a block diagonal matrix
because it acts on each factor separately.
Therefore $I+\bs{\iota}(\bs{u})$ is also a block diagonal matrix.
Each block is indexed by $\alpha \in F$ and denoted by $U_{\alpha}$.
Thus, for $\alpha=\{i_1,\ldots,i_{d_{\alpha}}\}$,
\begin{small}
\begin{equation}
 U_{\alpha}=
\begin{bmatrix}
 I_{r_{i_1}}    & u^{\alpha}_{i_2 \rightarrow i_1} & \cdots     &  u^{\alpha}_{i_{d_{\alpha}} \rightarrow i_1} \\
 u^{\alpha}_{i_1 \rightarrow i_2}   &  I_{r_{i_2}}         &  \cdots   & u^{\alpha}_{i_{d_{\alpha}} \rightarrow i_2} \\
   \vdots              &    \vdots      &  \ddots          & \vdots \\
u^{\alpha}_{i_{d_1 \rightarrow i_{d_{\alpha}}}} &  u^{\alpha}_{i_{d_2 \rightarrow i_{d_{\alpha}}}} &  \cdots & I_{r_{i_{d_{\alpha}}}}\\
\end{bmatrix}. \label{eq:defU}
\end{equation}
\end{small}
We also define $w^{\alpha}_{\ed{i}{j}}$ by the elements of $W_{\alpha}=U_{\alpha}^{-1}$:
\begin{small}
\begin{equation}
 W_{\alpha}=
\begin{bmatrix}
w^{\alpha}_{i_1 \rightarrow i_1}  & w^{\alpha}_{i_2 \rightarrow i_1} & \cdots     & w^{\alpha}_{i_{d_{\alpha}} \rightarrow i_1} \\
w^{\alpha}_{i_1 \rightarrow i_2}  & w^{\alpha}_{i_2 \rightarrow i_2} &  \cdots    & w^{\alpha}_{i_{d_{\alpha}} \rightarrow i_2} \\
   \vdots              &    \vdots      &  \ddots          & \vdots \\
w^{\alpha}_{i_{d_1 \rightarrow i_{d_{\alpha}}}} &  w^{\alpha}_{i_{d_2 \rightarrow i_{d_{\alpha}}}} &  \cdots
& w^{\alpha}_{i_{d_{\alpha}} \rightarrow i_{d_{\alpha}}    }  \\
\end{bmatrix}. \label{eq:defW}
\end{equation}
\end{small}

Similar to the definition of $\vfe$ in Subsection \ref{sec:defgraphzeta},
we define $\vfv$ as the set of functions on $V$ that takes value on $\mathbb{C}^{r_i}$ for each $i \in V$.

\begin{thm}[Determinant formula of Ihara-Bass type]
\label{thm:Ihara}
Let $\mathcal{D}$ be $\mathcal{W}$ are linear transforms on $\vfv$
defined by
\begin{equation}
 (\mathcal{D}g)(i):= d_i g(i), \qquad
 (\mathcal{W}g)(i):= \sum_{{e,e' \in \vec{E} \atop {t(e)=i, s(e)=s(e')}  }}
w^{s(e)}_{\ed{t(e')}{i}} g(t(e')).  \label{eq:defDW}
\end{equation}
Then, we have the following formula
\begin{equation}
  \zeta_{G}(\bs{u})^{-1}=
\det \big(
I_{r_V}- \mathcal{D} + \mathcal{W}
\big)
\prodf \det U_{\alpha},  \nonumber
\end{equation}
where $r_V:=\sum_{i \in V}r_i$
\end{thm}
The proof is given in Appendix~\ref{app:DetailProofs}

\subsection{\IB type determinant formula on ordinary graphs}
\label{sec:specialIB}
In this subsection, we explicitly write definitions and the above formula for better understanding.
A hypergraph $H=(V,F)$, which has only hyperedges of degrees two,
is naturally identified with an (undirected) graph $G_{H}=(V,E)$.
In the next section, we see that this case corresponds to the pairwise \ifa.

First, we define the zeta function $Z_G$ of a graph $G=(V,E)$.
For each undirected edge, we make a pair of oppositely
directed edges, which form a set of {\it directed edges} $\vec{E}$.
Thus $|\vec{E}|=2|E|$. For each directed edge $e \in \vec{E}$, $o(e)
\in V$ is the {\it origin} of $e$ and $t(e) \in V$ is the {\it
terminus} of $e$.   For $e \in \vec{E}$, the {\it inverse edge} is
denoted by $\bar{e}$, and the corresponding undirected edge by
$[e]=[\bar{e}] \in E$.

A {\it closed geodesic} in $G$ is a sequence $(e_1,\ldots,e_k)$ of
directed edges such that
$t(e_i)=o(e_{i+1}), e_i \neq \bar{e}_{i+1}$ for $i \in \mathbb{Z}/k\mathbb{Z}$.
Prime cycles are defined in the same manner to that of hypergraphs.
The set of prime cycles is denoted by $\mathfrak{P}_{G}$.

\begin{defn}
Let $G=(V,E)$ a graph.
For given positive integers $\{r_i\}_{i \in V}$ and
matrix weights $\bsu=\{ u_{e} \}_{e \in \vec{E}}$ with 
$u_e \in \mat{r_{t(e)}}{r_{o(e)}}$,
 \begin{equation}
Z_{G}(\boldsymbol{u}):=\prod_{\mathfrak{p} \in \mathfrak{P}_G }
\det(1-\pi(\mathfrak{p}))^{-1},
\quad
\pi(\mathfrak{p}):=u_{e_1} \cdots u_{e_k}
\quad
\text{ for }
\mathfrak{p}=(e_1,\ldots,e_k). \nonumber
\end{equation}
\end{defn}
Since $\mathfrak{P}_{\ug{H}}$ is naturally identified with $\mathfrak{P}_{H}$,
$Z_{\ug{H}}=\zeta_{H}$ holds.
This zeta function is the matrix weight extension of the edge zeta function analyzed by \citet{STzeta1}.

Since the degree of every hyperedge is equal to two for a graph,
the matrix $W_{\alpha}$ defined in Eq.~(\ref{eq:defW}) has explicit expressions.
Using this fact, we obtain the following simplification of Theorem~\ref{thm:Ihara}.
\begin{cor}
\label{cor:IBfornonhyper}
For a graph $G=(V,E)$,
 \begin{equation}
  Z_{G}(\bs{u})^{-1}=
  \det ( I + \hat{\mathcal{D}}(\bsu) - \hat{\mathcal{A}}(\bsu) )
  \prod_{[e] \in {E}} \det(I - u_e u_{\bar{e}}), \label{eq:IBfornonhyper}
 \end{equation}
where $\hat{\mathcal{D}}$ and $\hat{\mathcal{A}}$ are defined by
\begin{align}
 &(\hat{\mathcal{D}}(\bsu)g)(i):= \Big( \sum_{e: t(e)=i}(I_{r_i}-u_eu_{\bar{e}})^{-1}u_eu_{\bar{e}} \Big)g(i), \label{eq:modDop} \\
 &(\hat{\mathcal{A}}(\bsu)g)(i):= \sum_{e: t(e)=i}(I_{r_i}-u_eu_{\bar{e}})^{-1}u_e g(o(e)). \label{eq:modAop}
\end{align}
\end{cor}
\begin{proof}
For $e=(\edij)$, the $U_{[e]}$ block is given by
\begin{equation}
 U_{[e]}=
\begin{bmatrix}
 I_{r_i} & u_e \\
u_{\bar{e}} & I_{r_j} \\
\end{bmatrix} \nonumber
\end{equation}
Therefore $\det  U_{[e]}= \det(I_{r_i}-u_eu_{\bar{e}})$ and
the inverse $W_{[e]}$ is
\begin{equation}
 W_{[e]}=
\begin{bmatrix}
(I_{r_i}-u_eu_{\bar{e}})^{-1} & 0 \\
0 & (I_{r_j}-u_{\bar{e}}u_{e})^{-1} \\
\end{bmatrix}
\begin{bmatrix}
 I_{r_i} & -u_e \\
-u_{\bar{e}} & I_{r_j} \\
\end{bmatrix}. \nonumber
\end{equation}
Plugging these equations into Theorem \ref{thm:Ihara}, we obtain the assertion.
\end{proof}

\citet{MSweighted,HSTweighted} have derived a weighted graph version of \IB type determinant formula under 
assumption that the scalar weights $\{u_e\}$ satisfy conditions $u_e u_{\bar{e}} = u^2$.
In this case, the factors $(1-u_e u_{\bar{e}})^{-1}$ in Eqs.~(\ref{eq:modDop},\ref{eq:modAop}) do not depend on $e$ and
Eq.~(\ref{eq:IBfornonhyper}) is further simplified.
Corollary \ref{cor:IBfornonhyper} gives the extension of the result to graphs with arbitrary weights.
A direct proof of Corollary \ref{cor:IBfornonhyper}, without discussing hypergraphs, is found in
the supplementary material of \citet{WFzeta}.

If all the weights are set to $u$,
the result reduces to the following formula
known as {\it Ihara-Bass formula}:
\begin{equation*}
 Z_{G}(u)^{-1}
=(1-u^2)^{|E|-|V|}
\det(I-u\mathcal{A}+ u^2 (\mathcal{D}-I) ),
 \end{equation*}
where
$\mathcal{D}$ is the {\it degree matrix} defined by
$\mathcal{D}_{i,j}= d_i \delta_{i,j}$, 
and
$\mathcal{A}$ is the {\it adjacency matrix}
by 
\begin{equation}
 \mathcal{A}_{i,j}=
\begin{cases}
1& \quad \text{if} \quad \{i,j\} \in E \\
0& \quad \text{otherwise.}
\end{cases} \nonumber
\end{equation}
Many authors have discussed the proof of the \IB formula.
The first proof was given by \citet{Bass}. See \citet{KSzeta,STzeta1} for others.
A combinatorial proof is given by \citet{FZcombinatorial}.

\subsection{Positive definiteness condition}\label{sec:zetapositive}
The \IB type determinant formula relates the matrices $\matmu$ and $(I_{r_V}- \mathcal{D} + \mathcal{W})$.
In the later sections, we see that $\matmu$ corresponds to the derivative of the LBP update
and $(I_{r_V}- \mathcal{D} + \mathcal{W})$ is closely related to the Hessian of the \Bfe function.

The following theorem is fundamental to prove Theorem \ref{thm:positive}.
\begin{thm}
\label{thm:zetapositive}
Assume that $\bsu=\{u^{\alpha}_{\edij}\}_{\alpha \in F, i,j \in \alpha}$
satisfies
$u^{\alpha}_{\edij} = u^{\alpha}_{\edji}$
and $\norm{u^{\alpha}_{\edij}} <1$, where $\norm{ \cdot }$ is an arbitrary operator norm.
If
$\spec{ \matmu } \hspace{0.5mm} \subset \hspace{0.5mm} \mathbb{C} \smallsetminus \mathbb{R}_{\geq 1}$, where $\spec{\cdot }$ denotes the set of eigenvalues, 
then $(I_{r_V}- \mathcal{D} + \mathcal{W})$ is a positive definite matrix.
\end{thm}
\begin{proof}
From the assumption of symmetry, $W_{\alpha}$ in Eq.~(\ref{eq:defW}) is a symmetric matrix.
Therefore, $\mathcal{W}$, defined in Eq.~(\ref{eq:defDW}), is also symmetric.
To prove the positive definiteness, we define $u^{\alpha}_{\edij}(t):=t u^{\alpha}_{\edij} \quad (t \in [0,1])$, 
which implies $\mathcal{M}(\bsu(t))= t \mathcal{M}(\bsu)$.
From the assumption, $U_{\alpha}(t)$ is invertible and thus $W_{\alpha}(t)$ is well defined for all $t$.
If $t=0$, $W_{\alpha}(0)= \mathcal{D}$
and $I- \mathcal{D} + \mathcal{W}(0)= I$ is obviously positive definite.
Since the eigenvalues of a symmetric matrix are real and continuous with respect to its entries,
it is enough to prove that
$\det(I- \mathcal{D} + \mathcal{W}(t)) \neq 0 $ on the interval $[0,1]$.
Under the condition on the eigenvalues of $\matmu$, $\det (I - \mathcal{M}(\bsu(t))) \neq 0$ holds for $t \in [0,1]$.
Therefore, Theorem~\ref{thm:Ihara} implies the claim.
\end{proof}

\section{Main theoretical results}
\label{sec:key}

In this section, we establish the connection between
the graph zeta function and the \Bfe function.
These results form a basis of the analyses in later sections.

In Subsection \ref{sec:BZ}, we prove a formula using the \IB type determinant formula
proved in the previous section.
The formula shows a concrete relation between the \Bfe function and the graph zeta function.
In Subsection \ref{sec:PDC}, we give a condition that the Hessian of the \Bfe function is positive-definite.

\subsection{\Bzf}\label{sec:BZ}
In this subsection, we show that the determinant of the Hessian of the \Bfe function
is essentially equal to the reciprocal of the graph zeta function\footnote{An intuitive understanding of this result,
based on the Legendre duality of two types of the \Bfe functions is discussed by \citet{Wthesis}}.

In order to make the assertion clear, we first recall the definitions and notations.
Let $H=(V,F)$ be a hypergraph and
let $\mathcal{I}=\{ \mathcal{E}_{\alpha},\mathcal{E}_i\}$ be an \ifa on $H$.
Exponential families $\mathcal{E}_i$ and $\mathcal{E}_{\alpha}$ are given by sufficient statistics $\phi_i$
and $\fa{\phi}$ as discussed in Subsection~\ref{sec:expfamily}.
Furthermore, as discussed in Subsection~\ref{sec:BFE},
a point $\bseta=\{\pa{\eta},\eta_i\} \in L$ is identified with
a set of pseudomarginals $\beliefs$.
\begin{thm}
\label{thm:BZ}
At any point of $\bseta=\{\pa{\eta},\eta_i\} \in L$ the following equality holds.
\begin{equation*}
\zeta_{H}(\bsu)^{-1}
\hspace{-1mm}
=
\det (I -\matmu)
=
\det(\nabla^2 F)
\prod_{\alpha \in F}
\hspace{-0mm}
\det(\var{b_{\alpha}}{ \fa{\phi}   })
\prod_{i \in V}
\hspace{-0mm}
\det(\var{b_{i}}{\phi_{i}})^{1-d_i},
\end{equation*}
where
\begin{equation}
 u^{\alpha}_{\ed{i}{j}}:=
\var{b_j}{{\phi}_{j}}^{-1}
\cov{b_{\alpha}}{{\phi}_{j}}{{\phi}_{i}} \label{def:u}
\end{equation}
is an $r_j \times r_i$ matrix, and $\nabla^2 F$ is the Hessian matrix with respect to the coordinate $\{\pa{\eta},\eta_i\}$.
\end{thm}
Note that he Hessian $\nabla^2 F$ does not depend on the given compatibility functions $\Psi_{\alpha}$ because
those only affect linear terms in $F$, and thus 
the formula is a property of \ifa $\mathcal{I}$.
Note also that the determinants of variances in the formula are always positive, because we assume 
all the local exponential families $\mathcal{E}_{\alpha}$ and $\mathcal{E}_i$
have positive definite covariance matrices.

The proof is based on the Ihara-Bass type determinant formula;
we check that the Hessian $\nabla^2 F$ is related to the matrix $(I- \mathcal{D} + \mathcal{W})$
if weights has the form of Eq.~(\ref{def:u}).
The key condition satisfied on the set $L$ is $\var{b_{\alpha}}{\phi_i}=\var{b_i}{\phi_i}$.

\begin{proof}
From the definition of the \Bfe function Eq.~(\ref{defn:Bfe}),
the (V,V)- block of $\Hesse F$ is given by
\begin{equation*}
 \pdseta{F}{i}{i}= \sum_{\alpha \ni i} \pdseta{\varphi_{\alpha}}{i}{i}+(1-d_i)\pdseta{\varphi_i}{i}{i},
\quad
 \pdseta{F}{i}{j}= \sum_{\alpha \supset \{i,j\}} \pdseta{\varphi_{\alpha}}{i}{j} \quad (i \neq j).
\end{equation*}
The (V,F)-block and (F,F)-block are given by
\begin{equation*}
 \pds{F}{\eta_i}{\pa{\eta}}= \pds{\varphi_{\alpha}}{\eta_i}{\pa{\eta}},
\quad \quad
 \pds{F}{\pa{\eta}}{\pb{\eta}}= \pds{\varphi_{\alpha}}{\pa{\eta}}{\pb{\eta}}\delta_{\alpha,\beta}.
\end{equation*}
Using the diagonal blocks of (F,F)-block, we erase (V,F)-block and
(F,V)-block of the Hessian by Gaussian elimination.
In other words, we choose a square matrix $X$ such that $\det X =1$ and
\begin{equation}
X^T (\nabla^2 F) X
=
\begin{bmatrix}
\quad Y & 0 \\
\quad 0 &
\Big( \pds{F}{\pa{\eta}}{\pb{\eta}} \Big)
\end{bmatrix},   \nonumber
\end{equation}
in which 
\begin{align}
 Y_{i,i}
&=\sum_{\alpha \ni i} \left\{ \pdseta{\varphi_{\alpha}}{i}{i}
- \pds{\varphi_{\alpha}}{\eta_i}{\pa{\eta}} \left(\pds{\varphi_{\alpha}}{\pa{\eta}}{\pa{\eta}}\right)^{-1}
\pds{\varphi_{\alpha}}{\pa{\eta}}{\eta_i}  \right\}  + (1-d_i) \pdseta{\varphi_{i}}{i}{i}, \label{eq:Yii}\\
Y_{i,j}
&=\sum_{\alpha \supset \{i,j\} } \left\{ \pds{\varphi_{\alpha}}{\eta_i}{\eta_j}
-\pds{\varphi_{\alpha}}{\eta_i}{\pa{\eta}} \left(\pds{\varphi_{\alpha}}{\pa{\eta}}{\pa{\eta}}\right)^{-1}
\pds{\varphi_{\alpha}}{\pa{\eta}}{\eta_j}  \right\} .  \label{eq:Yij}
\end{align}
On the other hand,
since $u^{\alpha}_{\ed{i}{j}}:=$ $\var{b_j}{{\phi}_{j}}^{-1}$ $\cov{b_{\alpha}}{{\phi}_{j}}{{\phi}_{i}}$,
the matrix $U_{\alpha}$ defined in Eq.~(\ref{eq:defU}) is
\begin{equation}
 U_{\alpha}
=
\diag ( \var{}{\phi_i}^{-1} | i \in \alpha)
\hspace{1mm} \var{b_{\alpha}}{(\phi_i)_{i \in \alpha}}.
\end{equation}
Since the matrix $\var{b_{\alpha}}{(\phi_i)_{i \in \alpha}}$ is a submatrix of $\var{b_{\alpha}}{\fa{\phi}}$,
its inverse can be expressed by submatrices of
$\var{b_{\alpha}}{\fa{\phi}}^{-1}= \pds{\varphi_{\alpha}}{\fa{\eta}}{\fa{\eta} }$
using the Schur complement formula, which 
shows that the elements of $W_{\alpha}=U_{\alpha}^{-1}$ is given by
\begin{equation}
\label{eq:wij}
 w^{\alpha}_{\edji}=
\left\{ \pdseta{\varphi_{\alpha}}{i}{j}
-\pds{\varphi_{\alpha}}{\eta_i}{\pa{\eta}} \left(\pds{\varphi_{\alpha}}{\pa{\eta}}{\pa{\eta}} \right)^{-1}
\pds{\varphi_{\alpha}}{\pa{\eta}}{\eta_j}  \right\}
\var{}{\phi_j}.
\end{equation}
It follows from Eq.~(\ref{eq:Yii}),(\ref{eq:Yij}) and (\ref{eq:wij}) that 
\begin{equation}
 Y \hspace{1mm} \diag \left( \var{}{\phi_i} | i \in V \right)= I - \mathcal{D} + \mathcal{W}, \nonumber
\end{equation}
where $\mathcal{D}$ and $\mathcal{W}$ are defined in Eq.~(\ref{eq:defDW}).
Accordingly, we obtain
\begin{align*}
 \zeta_H(\bsu)^{-1}
&= \det (I-\mathcal{D}+\mathcal{W}) \prodf \det U_{\alpha}  \\
&= \det Y \prodv \det ( \var{}{\phi_i} ) \prodf
\frac{ \det \left( \var{b_{\alpha}}{(\phi_i)_{i \in \alpha}} \right)}{ \prod_{j \in \alpha} \det \left( \var{}{\phi_j} \right) } \\
&= \det \left( \Hesse F \right) \prodv \det ( \var{}{\phi_i} )^{1-d_i} \prodf
\frac{ \det \left( \var{b_{\alpha}}{(\phi_i)_{i \in \alpha}} \right)}{\det \left( \pds{\varphi_{\alpha}}{\pa{\eta}}{\pa{\eta}} \right)} \\
&= \det \left( \nabla^2 F \right) \prodf
\det(\var{b_{\alpha}}{ \fa{\phi} })
\prodv \det(\var{b_{i}}{{\phi}_{i}})^{1-d_i},
\end{align*}
where 
$ \det \left( \var{b_{\alpha}}{(\phi_i)_{i \in \alpha}} \right) \det \left( \pds{\varphi_{\alpha}}{\pa{\eta}}{\pa{\eta}} \right)^{-1}$
$=\det \left( \var{}{\fa{\phi}}   \right)$
is used. 
\end{proof}

In the rest of this subsection,
we rewrite the \Bzf in some specific cases.
Especially, we give explicit expressions of the determinants of the variance matrices.

\subsubsection*{Case 1: Multinomial \ifa}%
First, we consider the multinomial case.
If we take the sufficient statistics of multinomial exponential family
as in Example \ref{example:multinomial},
the determinant of the variance is
\begin{equation}
  \det \left( \var{p}{{\phi}} \right)= \prod_{k=1}^{N} p(k). \nonumber %
\end{equation}
Therefore, the theorem reduces to the following form \footnote{Here, we ignore minor constant factors
which come from the choices of sufficient statistics}.
\begin{cor}[Bethe-zeta formula for multinomial \ifa]
~\\
For any $\beliefsw \in L$ the following equality holds.
 \begin{equation*}
\zeta_{G}(\bsu)^{-1}
\hspace{-1mm}
=
\det(\nabla^2 F)
\prod_{\alpha \in F}
\prod_{x_{\alpha}} b_{\alpha}(x_{\alpha})
\prod_{i \in V}
\prod_{x_i} b_i(x_i)^{1-d_i},
 \end{equation*}
where
$u^{\alpha}_{\ed{i}{j}}:=$
$\var{b_j}{{\phi}_{j}}^{-1}$
$\cov{b_{\alpha}}{{\phi}_{j}}{{\phi}_{i}}$
is an $(N_j -1)  \times (N_i -1)$ matrix.
\end{cor}
For binary and pairwise case, this formula is first shown by \citet{WFzeta}.

\subsubsection*{Case 2: Fixed-mean Gaussian \ifa}%
Let $G=(V,E)$ be a graph.
We consider the fixed-mean Gaussian \ifa on $G$.
For a given vector $\bsmu=(\mu_i)_{i \in V}$, the \ifa
is constructed from sufficient statistics
$
\phi_i(x_i)=(x_i-\mu_i)^2  \quad \text{ and } \quad \pij{\phi}(x_i,x_j)=(x_i-\mu_i)(x_j-\mu_j).
$
Their expectation parameters are denoted by $\eta_{ii}$ and $\eta_{ij}$, respectively.
The variances and covariances are
\begin{equation}
\var{}{\phi_i}=2 \eta_{ii}^2 , \quad
\var{}{\phi_{\{i,j\}}}=
\begin{small}
\begin{bmatrix}
 2 \eta_{ii}^2 & 2 \eta_{ij}^2 & 2 \eta_{ii} \eta_{ij} \\
 2 \eta_{ij}^2 & 2 \eta_{jj}^2 & 2 \eta_{jj} \eta_{ij} \\
 2 \eta_{ii} \eta_{ij} & 2 \eta_{jj}\eta_{ij} &  \eta_{ij}^2 +\eta_{ii} \eta_{jj}\\
\end{bmatrix},
\end{small}   \nonumber
\end{equation}
where
$
\phi_{\{i,j\}}(x_i,x_j)= \left( (x_i-\mu_i)^2,(x_j-\mu_j)^2,(x_{i}-\mu_i)(x_j-\mu_j) \right).
$
Therefore, $\det(\var{}{\phi_{\{i,j\}}})=4 (\eta_{ii} \eta_{jj} - \eta_{ij}^2)^3 $.
\begin{cor}\label{cor:fixed-meanBzf}
[\Bzf for fixed-mean Gaussian \ifa]
For any $\{\eta_{ii},\eta_{ij}\} \in L$ the following equality holds.
 \begin{equation*}
Z_{G}(\bsu)^{-1}
\hspace{-1mm}
=
\det(\nabla^2 F)
\prodv \eta_{ii}^{2(1-d_i)}
\prod_{ij \in E} (\eta_{ii} \eta_{jj} - \eta_{ij}^2)^3
\hspace{2mm} 2^{|V|,}
 \end{equation*}
where
$u^{}_{\ed{i}{j}}:= \eta_{ij}^2 \eta_{jj}^{-2}$
is a scalar value.
\end{cor}
One interesting point of this case is that the edge weights $u_{\ed{i}{j}}$ are
always positive.

\subsection{Positive definiteness condition}
\label{sec:PDC}
In this section, we derive a condition that guarantees the positive-definiteness
of the Hessian of the \Bfe function.
It is based on Theorem \ref{thm:zetapositive}, which gives
a condition that the matrix $(I - \mathcal{D} + \mathcal{W})$ is positive definite in terms of the matrix $\matmu$.
As we have seen in the proof of the previous theorem,
$\nabla^2 F$ and $(I - \mathcal{D} + \mathcal{W})$ are essentially the same.
Thus, we obtain the following theorem.

\begin{thm}
\label{thm:positive}
Let $\bsu$ be given by $\bseta \in L$ using Eq.~(\ref{def:u}).
Then,
\begin{equation*}
\spec{ \matmu } \hspace{0.5mm} \subset
\hspace{0.5mm} \mathbb{C} \smallsetminus \mathbb{R}_{\geq 1}
\quad \Longrightarrow \quad
\nabla^2 F (\bseta)
\text{ is a positive definite matrix.}
\end{equation*}
\end{thm}

Before the proof of the theorem,
we remark the following fact.
It implies that we can change the matrix weight to the \ccms.
\begin{lem}
\label{lem:specutoc}
Let $\bseta$ be a point in $L$, $u^{\alpha}_{\edij}$ be given by Eq.~(\ref{def:u}), and 
\begin{equation}
c^{\alpha}_{\edij}:= \corr{b_{\alpha}}{\phi_j}{\phi_i} =
\var{b_{\alpha} }{\phi_j}^{-1/2} \cov{b_{\alpha}}{\phi_j}{\phi_i} \var{b_{\alpha}}{\phi_i}^{-1/2}  \nonumber
\end{equation}
be the \ccm, where $b_{\alpha}$ corresponds to $\bseta$.
Then
\begin{equation}
 \spec{\matmu }=\spec{\matmc}.
\end{equation}
\end{lem}
\begin{proof}
Define $\mathcal{Z}$ by $(\mathcal{Z})_{e,e'}:= \delta_{e,e'} \var{}{\phi_{t(e)}}^{1/2}$.
Then
\begin{equation}
(\mathcal{Z}\matmu \mathcal{Z}^{-1})_{e,e'}=
\var{}{\phi_{t(e)}}^{1/2}  \matmu_{e,e'} \var{}{\phi_{t(e')}}^{-1/2}
= \matmc_{e,e'}. \nonumber
\end{equation}
\end{proof}

\begin{proof}[Theorem \ref{thm:positive}]
By definition, $c^{\alpha}_{\edij}=c^{\alpha}_{\edji}$ holds.
We choose the operator norm induced by the inner product of the vector spaces.
In other words, $\norm{X}$ is equal to the maximum singular value of $X$.
In this case, it is well known that the norm of a \ccm is smaller than 1. From Theorem \ref{thm:zetapositive}, the matrix $(I - \mathcal{D} + \mathcal{W})$ for the weights $\bs{c}= \{c^{\alpha}_{\edij} \}$ is positive definite.

Next, we compute the matrix $(I - \mathcal{D} + \mathcal{W})$ for the weight $\bs{c}$.
Similar to Eq.~(\ref{eq:wij}), we obtain
\begin{equation}
 w^{\alpha}_{\edji}=
\var{}{\phi_i}^{1/2}
\left\{ \pdseta{\varphi_{\alpha}}{i}{j}
-\pds{\varphi_{\alpha}}{\eta_i}{\pa{\eta}} \left(\pds{\varphi_{\alpha}}{\pa{\eta}}{\pa{\eta}} \right)^{-1}
\pds{\varphi_{\alpha}}{\pa{\eta}}{\eta_j}  \right\}
\var{}{\phi_j}^{1/2} . \nonumber
\end{equation}
Therefore, using the same notations as in the proof of Theorem \ref{thm:BZ},
we have
\begin{equation}
\diag \left( \var{}{\phi_i} | i \in V \right)^{1/2} \hspace{1mm} Y \hspace{1mm} \diag \left( \var{}{\phi_i} | i \in V \right)^{1/2}=
(I - \mathcal{D} + \mathcal{W}). \nonumber
\end{equation}
This equation implies that $Y$ and $\nabla^2 F (\bseta)$ are positive definite.
\end{proof}

To check the condition of Theorem \ref{thm:positive},
we need to analyze the extent of the eigenvalues.
An easy way for narrowing down the possible region is to bound the spectral radius.
For a given square matrix $X$, the spectral radius of $X$ is
the maximum of the modulus of the eigenvalues;
it is denoted by $\specr{X}$.
The following proposition provides a useful bound.
The proof is given in Appendix \ref{app:DetailProofs}.

\begin{prop}
\label{prop:specradiusbound}
Let $\bs{u}=\{u^{\alpha}_{\edij}\}$ be arbitrary matrix weights and let
$\norm{\bs{u}}=\{ \norm{u^{\alpha}_{\edij}} \}$ be the scalar weights obtained by an arbitrary operator norm.
Then,
\begin{equation}
\specr{ \matmu } \leq \hspace{1mm} \specr{ \hspace{1mm} \mathcal{M}(\norm{\bsu}) \hspace{1mm} }
\leq \hspace{1mm} \max \norm{ u^{\alpha}_{\edij} } \hspace{1mm} \specr{\mathcal{M}}                                    \nonumber
\end{equation}
\end{prop}

\section{Analysis of positive definiteness and convexity of BFE}
\label{sec:pdconv}

The \Bfe function is not necessarily convex
though it is an approximation of the Gibbs free energy function, which is convex.
Non-convexity of the Bethe free energy can lead to multiple fixed points.
\citet{PAstat} and \citet{Huniquness} have derived sufficient conditions of the convexity and
shown that the Bethe free energy is convex for trees and graphs with one cycle.
In this section, not only such a global structure,
we shall focus on the local structure of the Bethe free energy function, i.e. the Hessian.
Our approach derives the region where the positive definiteness is broken.
All the results are based on the techniques developed in the previous section.

In Subsection \ref{sec:regionPD}, as an application of the positive definite condition,
we analyze the region where the Hessian of \Bfe function is positive definite.
The Hessian does not depend on the given compatibility function, $\Psi$,
because it appears in the linear part of the \Bfe function.
In Subsection \ref{sec:convrestricted}, we deal with the compatibility functions
by restricting the \Bfe function on a subset $S(\Psi)$ of $L$.
This set consists of the pseudomarginals that has \nparas $\{ \pa{ \bar{\theta} } \}$
and thus includes all the fixed point beliefs.
We will see that the problem of the uniqueness of the LBP fixed points  is reduced to the following problem:
is the subset $S(\Psi)$ included in the positive definite region of the original \Bfe function?

\subsection{Region of positive definite and convexity condition}
\label{sec:regionPD}
In this subsection, we simplify Theorem~\ref{thm:positive}
and explicitly see that if the \ccms of the pseudomarginals
are sufficiently small, then the Hessian is positive definite.
This ``smallness'' criteria depends on graph geometry.

In the following, we choose the operator norm that is equal to the maximum singular value.
It is well known that the norm of a \ccm is smaller than 1 under the assumption that the variance-covariance matrix is non-degenerate.

\begin{cor}[Positive definite region]
\label{cor:positivedefiniteregion}
Let $\kappa$ be the Perron-Frobenius eigenvalue of $\mathcal{M}$,  and
define
\begin{equation*}
 L_{\kappa^{-1}}(\mathcal{I}):=\left\{ \beliefsw \in L(\mathcal{I})~|~
\forall \alpha \in F, ~\forall i,j \in \alpha, \hspace{2mm}
\norm{\corr{b_{\alpha}}{\phi_i}{\phi_j}} < \kappa^{-1} \right\}.
\end{equation*}
Then, the
Hessian $\nabla^2 F$ is positive definite on $L_{\kappa^{-1}}(\mathcal{I})$.
\end{cor}
\begin{proof}
From Proposition~\ref{prop:specradiusbound} and $\hspace{1mm} \max \norm{ c^{\alpha}_{\edij} } \hspace{1mm} \kappa  < 1$,
 $\spec{ \mathcal{M}(\bs{c}) } \subset \{ \lambda \in \mathbb{C} |\hspace{2mm} |\lambda| < 1 \}.$
Therefore, from Theorem~\ref{thm:positive}, the Hessian is positive definite at the point.
\end{proof}

A bound of the Perron-Frobenius eigenvalue of $\mathcal{M}$ is
given in Subsection~\ref{sec:miszeta}.
Roughly speaking, as the degrees of factors and vertices increase,
$\kappa$ also increases and thus $L_{\kappa^{-1}}$ shrinks.
The Perron-Frobenius eigenvalue is equal to
$0$ (resp. $1$) if the hypergraph is a tree (resp. has a unique cycle).
This result suggests that LBP works better for graphs of low degree.

The convexity of $F$ depends solely on the given \ifa and the underlying hypergraph, 
because the Hessian $\Hesse F$ does not depend on the given compatibility functions, $\Psi=\{ \Psi_{\alpha}\}$.
For multinomial case, \citet{PAstat} have shown that the \Bfe function is convex if the hypergraph has
at most one cycle.
The following theorem extends the result.
To show the direction of $(i)$ we have only to analyze the \Bfe function on trees and one-cycle hypergraphs.
To show $(ii)$, however, we need to capture the effect of cycles on arbitrary hypergraphs.

\begin{thm}
 \label{thm:convexcondition}
Let $H$ be a connected hypergraph. \\
(i) If $n(H)=0 \text{ or } 1$, then $F$ is convex on $L$. \\
(ii) Assuming the \ifa is either a multinomial, Gaussian or fixed-mean Gaussian, then the converse of (i) holds.
\end{thm}
\begin{proof}
$(i)$
As we have mentioned above,
the Perron-Frobenius eigenvalue $\alpha$ of $\mathcal{M}$ is equal to $1$ if $n(H)=0$
and $0$ if $n(H)=1$.
Using Corollary~\ref{cor:positivedefiniteregion},
we obtain $L_{\alpha^{-1}}=L$.
Therefore, the \Bfe function is convex over the domain $L$.\\
$(ii)$ %
Here we show the proof only in the case of fixed-mean Gaussian.
(Other cases are proved by a similar way in Appendix \ref{app:DetailProofs}.)
Let $G=(V,E)$ be a graph.
For $t \in [0,1)$, define $\eta_{ii}(t):=1$ and $\eta_{ij}(t):=t$.
Accordingly, $u^{ij}_{\edij}=t^2$ and $\bseta(t) \in L$.
As $t \nearrow 1$, $\bseta(t)$ approaches to a boundary point of $L$.
From Theorem \ref{thm:Hashimoto} in Appendix \ref{sec:miszeta},
\begin{align*}
 \det(\nabla^2 F(t)) (1-t^2)^{2|E|+|V|-1}
&= 2^{-|V|}
Z_{G}(t^2)^{-1}
(1-t^2)^{-|E|+|V|-1} \\
&
\longrightarrow
-2^{|E|-2|V|+1}(|E|-|V|) \kappa(G) \quad (t \rightarrow 1).
\end{align*}
If $n(G)=|E|-|V|+1 > 1$, the limit value is negative.
Therefore, in a neighborhood of the limit point, $\Hesse F$ is not positive definite.\\
\end{proof}

\subsection{Convexity of restricted \Bfe function}
\label{sec:convrestricted}
Our analysis so far has not involved the given compatibility function,
because it disappears in the second derivatives.
Not only graph structure, however, but also the compatibility functions
affect the properties of LBP and the \Bfe.

In this section, we show a method for dealing with the compatibility functions.
We see that the understanding of the positive definite region helps us to deduce a uniqueness condition of LBP.

\subsubsection{Restricted \Bfe function}
First, we make simple observations.
Since beliefs are given by Eq.~(\ref{eq:defbelief1},\ref{eq:defbelief2}),
they must satisfy the following condition:
for each factor $\alpha$, there exists $\{\theta^{'\alpha}_i \}_{i \in \alpha}$ such that
\begin{equation*}
\quad
 b_{\alpha}(x_{\alpha}) \propto  \exp
(
\inp{\bar{\theta}_{\alpha}}{\phi_{\alpha}} + \sum_{i \in \alpha} \inp{ \theta^{'\alpha}_i}{\phi_i(x_i)}
),
\end{equation*}
where $\bar{\theta}_{\alpha}$ is the \npara of $\Psi_{\alpha}$. (See Eq.~(\ref{eq:asm:modelindludes}).)
In other words, we can say that all the beliefs are always in the following subset of $L$:
\begin{equation*}
 S(\Psi) := \left\{
\{\eta_{\alpha} , \eta_i\} \in L ~|
~{}^{\forall} \alpha \in F,\quad
\Lambda_{\alpha}^{-1} \pa{ (\eta_{\alpha})} = \pa{ \bar{\theta} }
~\right\}.
\end{equation*}
We can take a coordinate $\{\eta_i\}_{i \in V}$ of $S(\Psi)$ because $\eta_{\alpha}$ is determined by $\{\eta_i\}_{i \in \alpha}$.
We obtain a function by restricting $F$ on the set and taking $\eta_{\alpha}$ as arguments.
The function is called {\it restricted \Bfe function} and denoted by $\hat{F}$.
The following proposition says that the stationary points of $\hat{F}$ also correspond to the LBP fixed points.
(This fact can be also stated as ``the fixed points of LBP are the stationary points of the \Bfe function.'')

\begin{prop}\label{prop:restBFEandLBP}
\begin{equation*}
 \pd{\hat{F}(\{\eta_i\})}{\eta_j}=0  \quad {}^{\forall} j \in V \iff
 \{ \eta_i \} \in S(\Psi)   \text{ is an LBP fixed point.}
\end{equation*}
\end{prop}
\begin{proof}
Using the chain rule of derivatives, we have
\begin{equation}
  \pd{\hat{F}}{\eta_j}=  \pd{F}{\eta_j}+ \sum_{ \alpha \ni i} \pd{F}{\pa{\eta}} \pd{\pa{\eta}}{\eta_j}. \nonumber
\end{equation}
From the definition of $F$ and $S(\Psi)$,
we have $\pd{F}{\pa{\eta}}=0$ on the set $S(\Psi)$.
Therefore, all the derivatives of $F$ are equal to zero if and only if those of $\hat{F}$ are zero.
\end{proof}

\subsubsection{Convexity condition and uniqueness}
\label{sec:restrictedconvex}
In the following, we analyze the (strict) convexity of the restricted \Bfe function.
Our focus is multinomial models.
As a result, we provide a new condition that guarantee the uniqueness.
From Proposition~\ref{prop:restBFEandLBP}, the LBP fixed point is unique if $\hat{F}$ is strictly convex.

From the viewpoint of approximate inference, the uniqueness of LBP fixed point is a preferable property.
Since LBP algorithm is interpreted as the variational problem of the \Bfe function,
an LBP fixed point that correspond to the global minimum is believed to be the best one.
If we find the unique fixed point of the LBP algorithm, it is guaranteed to be the global minimum of $F$.

To understand the convexity of $\hat{F}$,
we analyze the Hessian.
It turns out that the positive definiteness of the Hessian of this function is equivalent
to the Hessian of $F$.
Note that the Hessian of $\hat{F}$ is of size ``$V$'' while that of $F$ is of size ``$V+F$''.
\begin{prop}
\label{prop:PDequiv}
At any points in the set $S(\Psi)$,
\begin{equation}
 \nabla^{2}\hat{F} \text{ is positive definite } \iff  \nabla^{2} F \text{ is positive definite.}  \nonumber
\end{equation}
\end{prop}
\begin{proof}
By taking the derivative of the equation $\pd{F}{\pa{\eta}}=0$ on the set $S(\Psi)$,
we obtain
\begin{equation*}
\pds{F}{\eta_i}{\pa{\eta}} + \sum_{ \beta} \pd{\pb{\eta}}{\eta_i} \pds{F}{\pb{\eta}}{\pa{\eta}} = 0.
\end{equation*}
This equation can be written as $(\pd{\pb{\eta}}{\eta_i})= - X_{F,F}^{-1} X_{F,V}$ using a notation
\begin{equation*}
\nabla^2 F
=
\begin{bmatrix}
~ X_{V,V} & X_{V,F} ~ \\
~ X_{F,V} & X_{F,F} ~
\end{bmatrix}.
\end{equation*}
A straightforward computation of the derivatives of $\hat{F}$ gives
$\nabla^{2}\hat{F}= X_{V,V} -  X_{V,F} X_{F,F}^{-1} X_{F,V}$, where
we used the above equation.
Since the block $X_{F,F}$ is always positive definite,
the statement is obvious.
\end{proof}

If we can verify that the set $S(\Psi)$ is in the region where
$\nabla^{2}F$ is positive definite,
we can show that $\hat{F}$ is convex.
Using Theorem~\ref{thm:positive}, we obtain the following.
\begin{thm}
\label{thm:restBFEConvex}
Define
\begin{equation}
 W^{\alpha}_{i,j}(\Psi) :=
\sup \Big\{ \hspace{1mm} \norm{ \corr{b_{\alpha}}{\phi_i}{\phi_j}} \hspace{1mm} |
\hspace{1mm} b_{\alpha}(x_{\alpha}) \propto \Psi_{\alpha}(x_{\alpha})
\prod_{i \in \alpha} f_i(x_i),  f_i \text{ are positive functions of }x_i   \Big\}. \nonumber
\end{equation}
If $\specr{ \mathcal{M}(\bs{W}) } < 1 $ then $\hat{F}$ is strictly convex.  Therefore, LBP has the unique fixed point.
\end{thm}
\begin{proof}
Let $\bseta$ be any point in $S(\Psi)$.
By definition, $\norm{ \corr{b_{\alpha}}{\phi_i}{\phi_j}}$ is smaller than $W_{i,j}^{\alpha}$.
From Theorem~\ref{thm:positive} and Proposition~\ref{prop:specradiusbound}, $\nabla^{2} \hat{F}$ is positive definite at the point.
\end{proof}

In principle, we can compute the weights, $W$, given the compatibility functions.
However, it requires optimizations with respect to $f$;
we can use standard numerical maximization techniques \citep{Venkataraman2009}. 
We leave developing efficient methods for computing the values for future works.
For binary pairwise case,
we have a useful formula $W^{\{i,j\}}_{i,j}(\Psi_{i,j})=\tanh(|J_{ij}|)$,
where $\Psi_{i,j}(x_i,x_j) \propto {\rm e}^{J_{ij}x_i x_j + h_i x_i + h_j x_j}$ \citep{WFzeta}.

Theorem \ref{thm:restBFEConvex} holds for an arbitrary LBP.
However, for Gaussian cases we obtain $W^{\{i,j\}}_{i,j}= 1$, yielding no meaningful implications.
Therefore, in the rest of this section, we focus on multinomial cases.

\subsubsection{Comparison to Mooij's condition}
For multinomial models, there are several works that give sufficient conditions for the uniqueness property.
\citet{Huniquness} analyzed the uniqueness problem by
considering an equivalent min-max problem.
Other authors analyzed the convergence property rather than the uniqueness.
LBP algorithm is said to be {\it convergent} if the messages converge to the unique fixed point
irrespective of the initial messages.
By definition, this property is stronger than the uniqueness.
\citet{TJgibbsmeasure} utilized the theory of Gibbs measure, and 
showed that the uniqueness of the Gibbs measure implies the convergence of
LBP algorithm.
Therefore, known sufficient conditions of the uniqueness of the Gibbs measure are
that of the convergence of LBP algorithm.
\citet{IFW} and \citet{MKsufficient} derived sufficient conditions
for the convergence by investigating conditions that make the LBP update a contraction;
for pairwise case, their conditions are essentially the same.

We compare our condition with Mooij's condition.
One reason is that this condition is directly applicable to factor graph models,
while Ihler's and Tatikonda's conditions are for written for pairwise models.
Another reason is that numerical experiments by \citet{MKsufficient} suggests that
Mooij's condition is far superior to the condition of Heskes.
(See numerical experiments in \citep{MKsufficient}.)

The Mooij's condition is stated as follows.

\begin{thm}[\cite{MKsufficient}]
\label{thm:MKsufficient}
Define
\begin{equation}
 N_{ij}(\Psi_{\alpha}) := \sup_{x_i \neq x_i'} \sup_{x_j \neq x_j'} \sup_{x_{\alpha_{ij}} \neq x_{\alpha_{ij}}'}
\tanh \Big(
\frac{1}{4} \log \frac{\Psi_{\alpha}(x_i,x_j,x_{\alpha_{ij}}) }{ \Psi_{\alpha}(x_i',x_j',x_{\alpha_{ij}}') }
\frac{\Psi(x_i',x_j,x_{\alpha_{ij}})}{ \Psi(x_i,x_j',x_{\alpha_{ij}}')}
\Big),
\quad
\alpha_{ij}= \alpha \setminus \{i,j\}.   \nonumber
\end{equation}
If  $\specr{ \mathcal{M}(\bs{N}) } < 1$, then LBP is convergent.
Therefore, LBP has a unique fixed point.
\end{thm}

Interestingly, this condition looks similar to our Theorem~\ref{thm:restBFEConvex};
both of them are stated in terms of the spectral radius of the directed edge matrix, $\mathcal{M}$, with weights.
Comparison of these condition is reduced to that of $W_{ij}(\Psi_{\alpha})$ and $N_{ij}(\Psi_{\alpha})$.
(Recall that for positive matrices $X$ and $Y$, $\specr{X} \leq \specr{Y}$ if $X_{ij} \leq Y_{ij}$. )
For binary pairwise case, the conditions coincide; it is not hard to check that
$W_{ij}=N_{ij}= \tanh(|J_{ij}|)$.

By numerical computation,
we conjecture that $W_{ij}(\Psi_{\alpha}) \leq N_{ij}(\Psi_{\alpha})$ always holds.
In Figure~\ref{fig:Final}, we show a plot for the case of
$\Psi(x_1,x_2,x_3)=\exp(K x_1 x_2 x_3 + 0.3 \sum x_i x_j)$,
where $x_i \in \{ \pm 1 \}$.
We observe that $W$ and $N$ coincides for large $|K|$, but $W$ is
strictly smaller than $N$ for small $|K|$.

Next, we compare conditions of Theorem~\ref{thm:restBFEConvex}, \ref{thm:MKsufficient} and the actual LBP convergence region.
We run the LBP algorithm on the $3 \times 3$ square grid of cyclic boundary condition, where the factors correspond to
the vertices of the grid and variables are on the edges.
Thus, the degree of factors is four and that of vertices is two.
The variables are binary $(x_i \in \{ \pm 1 \})$ and
compatibility functions are given in the form of $\Psi(x_1,x_2,x_3,x_4)= \exp( K \sum_{i<j<k} x_i x_j x_k + J \sum_{i<j} x_i x_j )$;
we changed the parameters $K$ and $J$.
All the messages are initialized to constant functions and updated in parallel by Eq.~(\ref{LBPupdate}).
The result is plotted in Figure \ref{fig:Final}.
We judge LBP is convergent if message change is smaller than $10^{-3}$ after $30$ iterations.
We observe that there is a triangle region where uniqueness is guaranteed but LBP does not converge.

\begin{figure}
\begin{minipage}{.45\linewidth}
 \begin{center}
\includegraphics[scale=0.6]{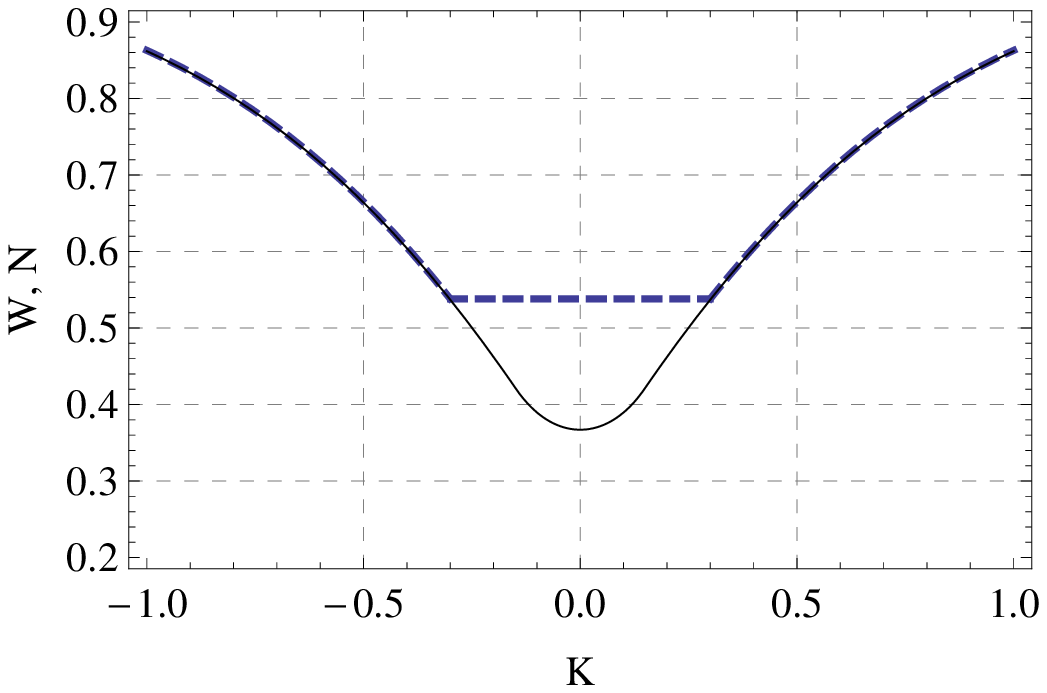}
\vspace{-1mm}
\end{center}
\end{minipage}
\begin{minipage}{.6\linewidth}
\begin{center}
\includegraphics[scale=0.5]{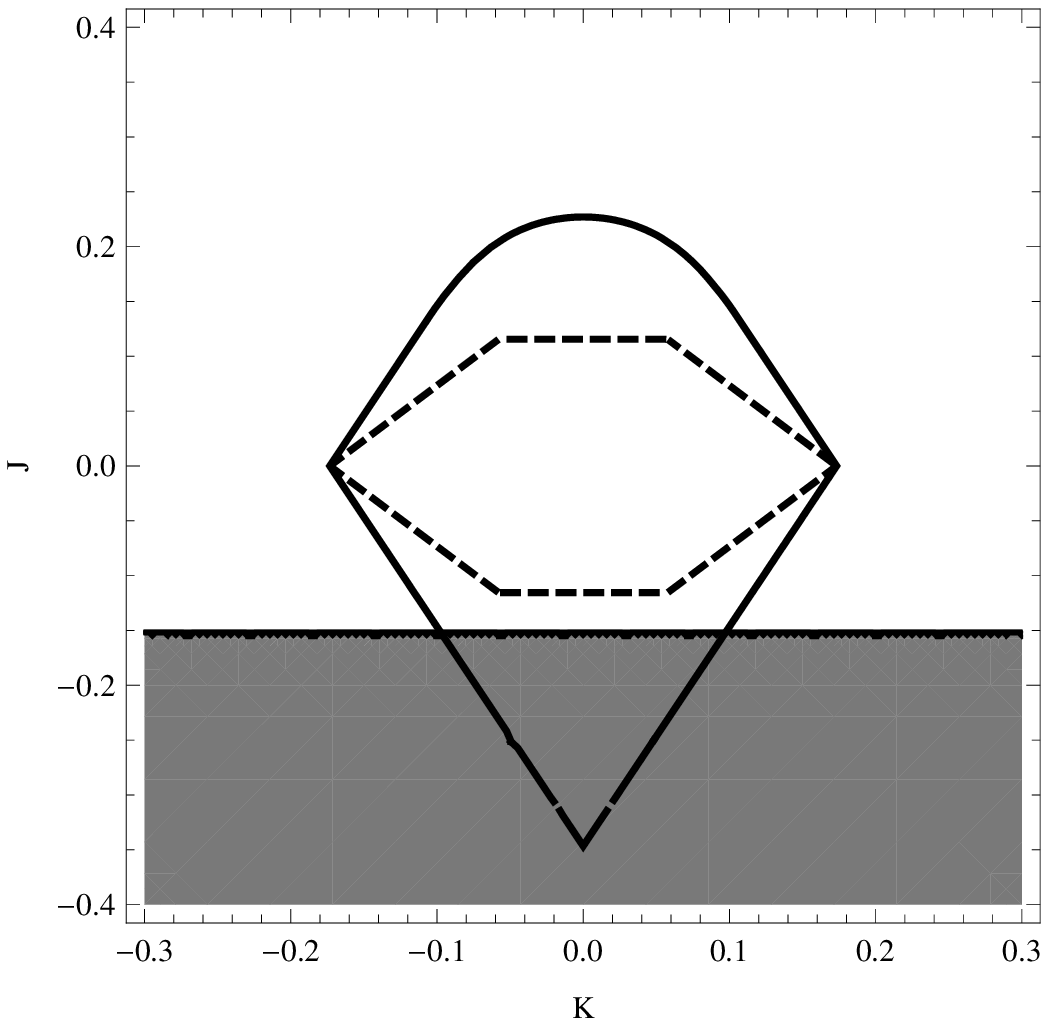}
\vspace{-1mm}
\end{center}
\end{minipage}
\caption{
Left: Comparison of $W$ and $N$. %
Solid line is the plot of $W$ and dashed line is $N$.
Right: Inside the dashed line region, LBP is guaranteed to converge by the Mooij's condition.
Inside the solid line, LBP is guaranteed to have the unique fixed point by Theorem~\ref{thm:restBFEConvex}.
In the shaded region, LBP does not converges.} \label{fig:Final}
\end{figure}

\section{Analysis of stability of LBP}
\label{sec:stab}

In this section, we analyze relations between the local stability of LBP and the local
structure of the Bethe free energy around an LBP fixed point.
Since LBP is not the gradient descent of the \Bfe function,
such a relation is not necessarily obvious.
From the view point of the variational formulation,
we hope to find the minima.
In the celebrated paper by \citet{YFWGBP}, %
they empirically found that locally stable LBP fixed points are local minima of the \Bfe function;
\cite{Hstable} have shown the fact for multinomial case.

In the following, we extend the result to two directions.
First, we derive the conditions of the local stability and local minimality
in terms of the eigenvalues of the matrix $\mathcal{M}(u)$, which  immediately implies the
above fact.
Secondly, the result is extended to LBPs formulated by \ifa including both multinomial and Gaussian cases.  
This is possible, since our analysis is based on the techniques developed in Section~\ref{sec:key}.

\subsection{LBP as a dynamical system}\label{sec:LBPstability}
First, we regard the LBP update as a dynamical system.
At each time $t$, the state of the algorithm is specified by the set of
messages $\{ m_{\edai}^t \}$, which is identified with
its \nparas $\bsmu^t=\{ \mu_{\edai}^t \} \in \mathbb{R}^{\vec{E}}$.
In terms of the parameters, the update rule Eq.~(\ref{LBPupdate}) is written as follows.
\begin{equation}
\label{eq:eparaLBPupdate}
 \mu_{\edai}^{{t+1}}=
\Lambda_{i}^{-1}
\Big(
   \Lambda_{\alpha} ( \bar{\theta}_{\alpha},
                \bar{\theta}^{\alpha}_{i_1} + \hspace{-2mm}  \sum_{\beta \in N_{i_1} \smallsetminus \alpha}\hspace{-2mm} \mu^t_{\edbione},
                \ldots, \bar{\theta}^{\alpha}_{i_k} + \hspace{-2mm}  \sum_{\beta \in N_{i_k} \smallsetminus \alpha}\hspace{-2mm} \mu^t_{\edbik}
                    )_i
\Big)
- \sum_{\gamma \in N_i \smallsetminus \alpha} \hspace{-2mm} \mu^t_{\edgi}, \nonumber
\end{equation}
where $\alpha = \fai$, $d_{\alpha}=k$ and
$\Lambda_{\alpha}(\cdots)_i$ is the $i$-th component ($i \in \alpha$).
To obtain this equation, after multiply Eq.~(\ref{LBPupdate}) by
\begin{equation*}
 \prod_{\gamma \in N_i \smallsetminus \alpha} m^t_{\edgi}(x_i), 
\end{equation*}
normalize it to be a probability density function, and then take the expectation of $\phi_i$.

Formally, this update rule can be viewed as a transform $T$ on the set of \nparas of messages $M$:
\begin{equation*}
 T:  M \longrightarrow  M, \qquad
 \bsmu^t = T( \bsmu^{t-1} ).
\end{equation*}
LBP algorithm can be formulated as repeated applications of this map.
In this formulation, the fixed points of LBP are $\{ \bsmu^* \in M |  \bsmu^*= T( \bsmu^* )\}$.

Here we compute the
differentiation of the update map $T$ around an LBP fixed point.
This expression derived by \citet{ITAinfo} for the cases of turbo and LDPC codes. %
\begin{thm}
[Differentiation of the LBP update]
\label{thm:diffofLBP}
At an LBP fixed point, the differentiation (linearization) of the LBP update is
\begin{equation}
 \pd{T(\bsmu)_{\edai}}{\mu_{\edbj}}=
\begin{cases}
\var{b_i}{\phi_i}^{-1} \cov{b_{\alpha}}{\phi_i}{\phi_j}
&\text{ if } j \in N_{\alpha} \smallsetminus i \text{ and } \beta \in N_j \smallsetminus \alpha, \\
0 &\text{ otherwise.}
\end{cases} \nonumber
\end{equation}
In other words, at an LBP fixed point $\bseta \in L$, the differentiation of $T$ is
\begin{equation}
 T' = \matmu, \nonumber
\end{equation}
where $\bsu=\{  u^{\alpha}_{\edij} \}$ is given by Eq.~(\ref{def:u}).
\end{thm}
\begin{proof}
First, consider the case that $j \in N_{\alpha} \smallsetminus i \text{ and } \beta \in N_j \smallsetminus \alpha$.
The derivative is equal to
\begin{equation}
 \pd{\Lambda^{-1}_i}{\eta_i}   \pd{(\Lambda^{}_{\alpha})_i}{\va{\theta}{j}}
=
\var{b_i}{\phi_i}^{-1} \cov{b_{\alpha}}{\phi_i}{\phi_j} . \nonumber
\end{equation}
Another case is $i=j$ and $ \alpha , \beta \in N_i ~(\alpha \neq \beta)$.
Then, the derivative is
\begin{equation}
 \pd{\Lambda^{-1}_i}{\eta_i}   \pd{(\Lambda^{}_{\alpha})_i}{\va{\theta}{i}}-I
=
0 \nonumber
\end{equation}
because $\var{b_i}{\phi_i}= \var{b_{\alpha}}{\phi_i}$ from Eq.~(\ref{eq:localconsistency}).
In other cases, the derivative is trivially zero.
\end{proof}

The relation $j \in N_{\alpha} \smallsetminus i \text{ and } \beta \in N_j \smallsetminus \alpha$
will be written as $(\edbj) \rightharpoonup (\edai)$ in Subsection \ref{sec:defgraphzeta}.
It is noteworthy that the elements of the linearization matrix is
explicitly expressed by the fixed point beliefs.

\subsection{Spectral conditions}
Let $T$ be the LBP update map.
A fixed point $\bsmu^{*}$ is called {\it locally stable}\footnote{ This property is often referred to as {\it asymptotically stable} \cite{GHnonlinear}.}
if LBP starting with a point sufficiently close to $\bsmu^{*}$ converges to $\bsmu^{*}$.
To suppress oscillatory behaviors of LBP,
{\it damping} of update $T_{\epsilon }:=(1- \epsilon ) T+ \epsilon I$
is sometimes useful, where $0 \leq \epsilon < 1$ is a damping strength
and $I$ is the identity matrix.

As we will summarize in the following theorem,
the local stability is determined by the linearization $T'$ at the fixed point.
Since $T'$ is nothing but $\matmu$ at an LBP fixed point,
Theorem~\ref{thm:positive} implies relations between the local stability and the Hessian of the \Bfe function.

\begin{thm}
\label{thm:speccondition}
Let $\bsmu^{*}$ be an LBP fixed point and assume that $T'(\bsmu^{*})$ has no eigenvalues of unit modulus for simplicity.
Then the following statements hold.
\begin{enumerate}
 \item $\spec{T'(\bsmu^{*})}  \subset \{ \lambda \in \mathbb{C}| |\lambda| < 1\}$ $\iff$ LBP is locally stable at $\bsmu^{*}$.
 \item $\spec{T'(\bsmu^{*})}  \subset \{ \lambda \in \mathbb{C}| {\rm  Re}\lambda < 1\}$ \hspace{-2mm} $\iff$ \hspace{-2mm}
       LBP is locally stable at $\bsmu^{*}$ with some damping.
 \item $\spec{T'(\bsmu^{*})}  \subset \mathbb{C} \smallsetminus \mathbb{R}_{\geq 1}$ $\Rightarrow$
       $\bsmu^{*}$ is a local minimum of the \Bfe function.
\end{enumerate}
\end{thm}
\begin{proof}
$1.:$ This is a standard result. (See \cite{GHnonlinear} for example.)
~$2.:$
There is an $\epsilon \in [0,1)$ that satisfy $\spec{T_{\epsilon}'(\bsmu^{*})}  \subset \{ \lambda \in \mathbb{C}| |\lambda| < 1 \}$
if and only if
$\spec{T'(\bsmu^{*})}  \subset \{ \lambda \in \mathbb{C}| {\rm  Re}\lambda < 1\}$.
~$3.:$
This assertion is a direct consequence of Theorem \ref{thm:positive} and \ref{thm:diffofLBP}.
\end{proof}

This theorem immediately implies that a 
locally stable LBP fixed point is a local minimum of the \Bfe.
The theorem applies to both the multinomial and Gaussian cases.

It is interesting to ask under which condition a local minimum of the
\Bfe function is a locally stable fixed point of (damped) LBP.
An implicit reason for the empirical success of the LBP algorithm is that
LBP finds a ``good'' local minimum rather than a local minimum nearby the initial point.
The theorem gives a new insight to the question, i.e., the difference between the 
stable local minima and the unstable local minima in terms of the spectrum of $T'(\bsmu^{*})$.

\subsection{Special cases: gaps between stability and local minimality}
Here we focus on two special cases: binary pairwise attractive models and
pairwise fixed-mean Gaussian models.
Note that a binary pairwise graphical model $\Psi=\{\Psi_{ij},\Psi_i\}$ is called {\it attractive} if $J_{ij}\geq 0$,
where  $\Psi_{i}(x_i)=\exp( h_i x_i )$ and $\Psi_{ij}(x_i,x_j)= \exp( J_{ij} x_i x_j)$ $~(x_i,x_j \in \{ \pm 1\})$.
In these cases, the stable fixed points of LBP and the local minima of \Bfe function are less different.

Consider the following situation: we have continuously parametrized compatibility functions $\{\Psi_{ij}(t),\Psi_{i}(t)\}_{t \geq 0}$,
which are constants at $t=0$
(e.g. $t$ is a inverse temperature: $\Psi_{ij}(t)=\exp(t J_{ij}x_i x_j)$ and $\Psi_{i}(t)=\exp(t h_i x_i )$).
Starting from $t=0$, we run LBP algorithm for $t$, find a stable fixed point 
and use it as initial messages of LBP for $t+\delta t$, where $\delta t$ is a sufficiently small positive number. 
Then we obtain a trajectory of a stable fixed point beliefs: we call it a {\it belief trajectory}.
It first continuously follow the local minima and then it may jump to another stable fixed point belief at $t=t_0$.
The following theorem implies that the stable fixed
point becomes unstable by continuous changes of the compatibility functions 
exactly when the corresponding local minimum becomes a saddle.

\begin{thm}
\label{thmattractive}
Suppose that we have a continuously parametrized compatibility functions 
of attractive binary pairwise model or fixed-mean Gaussian model as above.
If the LBP fixed point becomes unstable across $t=t_0$ for the first time following the belief trajectory,
then the corresponding local minimum of the Bethe free energy becomes
a saddle point across $t=t_0$.
\end{thm}
\begin{proof}
First consider the case of attractive binary pairwise models.
From Eq.~(\ref{eq:defbelief2}),
we see that
$b_{ij}(x_i,x_j) \propto \exp (J_{ij}x_i x_j + \theta_i x_i+ \theta_j x_j)$
for some $\theta_i$ and $\theta_j$.
From $J_{ij} \geq 0$,
we have $\cov{b_{ij}}{x_i}{x_j} \geq 0$, and thus $u_{\edij} \geq 0$.
When the LBP fixed point becomes unstable,
the Perron-Frobenius eigenvalue of $\matmu$ goes over $1$, which means
$\det(I-\matmu)$ crosses $0$.
From Theorem \ref{thm:BZ}, we see that
$\det(\nabla^{2} F)$ becomes positive to negative at $t=t_0$.
The Gaussian case can be proved analogously.
Recall that the weight $u_{\edij}$ are always positive scalars as shown in Corollary \ref{cor:fixed-meanBzf}.
\end{proof}

Theorem \ref{thmattractive} extends Theorem 2 of \cite{MKproperty}, which discusses only
the case of binary pairwise models with vanishing local fields $h_i=0$
and the trivial fixed point (i.e. $\E{b_i}{x_i}=0$).

\section{Summary and discussions}

We have established a connection between graph zeta function, \Bfe and \lbp.
We have shown that this connection provides powerful tools for the analysis of \Bfe and LBP;
key theorems are given in Section \ref{sec:key}.
In Section~\ref{sec:pdconv}, based on the theorems, we analyzed the (non) convexity of the \Bfe function.
Roughly speaking, the positive definite region of \Bfe functions shrinks as
the Perron-Frobenius eigenvalue of the directed edge matrix becomes large,
or equivalently, as the pole of the Ihara zeta function closest to the origin approaches to zero.
We have shown that such knowledge can be used to derive the uniqueness property of LBP.
In Section~\ref{sec:stab}, we have shown that the local stability of LBP implies local minimality of \Bfe
as long as LBP is well defined within a class of exponential families.
A key observation is that the matrix $\matmu$
is equal to the linearization of the LBP update at LBP fixed points.

The \Bzf shows that the \Bfe function contains information on the graph geometry, especially 
on the prime cycles.
The formula helps extract graph information from the \Bfe function.
For example we observed that the number of the spanning trees
are derived from a limit of the \Bfe function.
In a sense, the connection between those three objects seems to be natural
as all of them becomes ``trivial'' if the associated graph structure is a tree.
If the associated hypergraph is a tree, zeta function is equal to $1$, \Bfe function is equal to the 
Gibbs free energy function and LBP reduces to the original BP, which computes exact marginals in finite steps.

\subsection{Path forward}
In this subsection, we list a few directions of future researches going beyond the results of this paper.

In a sequel paper \citep{WFMathZeta}, we further exploit the connection
between LBP, \Bfe and graph zeta function
to analyze the LBP fixed point equation, focusing on binary pairwise models.
We characterize the class of signed graph on which uniqueness of the LBP fixed point is guaranteed.
Note that the signs on the edges represents those of the interactions (i.e. $\sgn J_{ij}$).
The condition is contrast to the those of the past researches and the result in Section~\ref{sec:pdconv},
where the strength of interactions (i.e. $|J_{ij}|$) are bounded.

In Subsection \ref{sec:convrestricted}, we have derived a 
condition for the convexity for the restricted \Bfe function.
Unfortunately, the expression of the weight $W$ involves $\sup$ operator and
does not easy to compute directly. 
We need further consideration to find a way of compute it more easily.  
The proof of the conjecture $W \leq N$ is also an interesting problem.

The connection between graph zeta, Bethe free energy and LBP can be extended to 
a more general class of free energies including fractional and tree-reweighted types \citep{WHfractional,WJWtree}.
These free energies are obtained by modifying the coefficients in the definition of the \Bfe function.
The corresponding graph zeta function then becomes the Bartholdi type, which allows cycles with backtracking \citep{Bcounting,Ibartholdi}.
The relation may be useful to analyze such class of free energies.

\acks{This work was supported in part by Grant-in-Aid for JSPS Fellows
20-993 and Grant-in-Aid for Scientific Research (C) 22300098.}

\newpage

\appendix
\section{}
\label{app:theorem}

\subsection{Miscellaneous properties of one-variable hypergraph zeta function}\label{sec:miszeta}
This subsection provides miscellaneous facts related to the one-variable hypergraph zeta functions.
In the analyses of this paper, we sometimes reduce the multivariate zeta to the one-variable zeta. 
Therefore, it is important to understand the one-variable hypergraph zeta $\zeta_{H}(u)$
and the directed edge matrix $\mathcal{M}$.

Recall that $\specr{X}$ denotes the spectral radius of $X$.
We have the following bounds on the spectral radius of $\mathcal{M}$.
\begin{prop}
\label{prop:PFboundM}
For $e \in \vec{E}$, let 
$k_e :=| \{ e'\in \vec{E}; \etea    \} |$,
$k_m= \min k_e$ and $k_{M}=\max k_e$. 
Then
\begin{equation*}
 k_m \leq \specr{\mathcal{M}} \leq k_M. \label{eq:luboundM}
\end{equation*}
Therefore, if $H$ is a graph,
\begin{equation}
 \min_{i \in V} d_i -1 \leq \specr{{\mathcal{M}}} \leq   \max_{i \in V} d_i -1.  \label{eq:PFboundMgraph}
\end{equation}
\end{prop}
\begin{proof}%
Since $k_e = \sum_{e'} \mathcal{M}_{e,e'}$, the bound is trivial from 
the easy bound on the spectral radius of non-negative matrices.
See Theorem 8.1.22 of \citet{HJmatrix}.
\end{proof}
Since the directed matrix $\mathcal{M}$ is non-negative, the spectral radius
is equal to the Perron-Frobenius eigenvalue.
The pole of $\zeta_{H}$ closest to the origin is $u= \specr{\mathcal{M}}^{-1} \geq k_M^{-1}$.
For the case of Ihara's zeta function, a bound on the modulus of imaginary poles as well as Eq.~(\ref{eq:PFboundMgraph})
are given by \citet{KSzeta}.

For arbitrary hypergraph, $\zeta_H(u)$ has a pole at $u=1$ 
because $\det ( I - \mathcal{M} ) =0$.
The following theorem gives the multiplicity of the pole.
The original version of this theorem is proved by \citet{Hpadic,Hzeta}.
\begin{thm}[Hypergraph Hashimoto's theorem \citep{Hpadic,Shypergraph}]
\label{thm:Hashimoto}
Let $\chi(H):=|V|+|F|-|\vec{E}|$ be the {\it Euler number} of $H$.
\begin{equation*}
 \lim_{u \rightarrow 0}
\zeta_H(u)^{-1}(1-u)^{- \chi(H)+1}= \chi(H) \kappa(B_H),
\end{equation*}
where $\kappa(B_H)$ is the number of spanning trees of the bipartite graph $B_H$.
($B_H$ is the bipartite graph representation of the hypergraph $H$.)
\end{thm}
\begin{proof}%
For a graph $G=(V,E)$, \citet{Hpadic,Hzeta} proved that
\begin{equation}
\lim_{u \rightarrow 1}
Z_{G}(u)^{-1}(1-u)^{-|E|+|V|-1}
=
-2^{|E|-|V|+1}(|E|-|V|)
\kappa(G), \nonumber
\end{equation}
where $\kappa (G)$ is the number of spanning tree of $G$.
A simple proof is given by \citet{Nnote}.
Since there is a one-to-one correspondence between
$\mathfrak{P}_H$ and $\mathfrak{P}_{B_H}$,
we have $\zeta_{H}(u)=Z_{B_H}(\sqrt{u})$.
Then the assertion is proved from the above formula.
\end{proof}

\subsection{Detailed Proofs}
\label{app:DetailProofs}
\begin{AppProof}{Theorem \ref{thm:LBPcharacterizations}}
The conditions for stationary points of the \Bfe function are
$\pa{\bar{\theta}} = \pa{\theta}$ and $\sum_{\alpha \ni i}(-\va{\bar{\theta}}{i}+\va{\theta}{i})+(1-d_i) \theta_i=0 $.\\
$(1. \Rightarrow 2.)$
The correspondence from the fixed point message to the stationary point is given by 
Eqs.~(\ref{eq:defbelief1},\ref{eq:defbelief2}).
From this construction, we see that
\begin{equation*} 
 \prodf \Psi_{\alpha}(x_{\alpha}) \propto
\prod_{\alpha}b_{\alpha}(x_{\alpha}) \prod_{i} b_i(x_i)^{1-d_i}.
\end{equation*}
This implies the above stationary point conditions.\\
$(2. \Rightarrow 1.)$
The converse correspondence is given by
$m_{\edai}(x_i) = \exp ( \inp{\theta_i + \va{\bar{\theta}}{i} - \va{\theta}{i} }{\phi_i} )$,
where $\thetasw$ are the natural parameters of the stationary point pseudomarginals $\beliefsw$.
From this construction and the stationary point conditions, we have
\begin{align*}
& \prod_{\beta \in N_i} m_{\edbi}(x_i) = \exp ( \inp{\theta_i}{\phi_i (x_i )} ) \propto b_i(x_i),\\
& \Psi_{\alpha} (x_{\alpha}) \prod_{i \in \alpha} \prod_{\beta \in N_i \smallsetminus \alpha} \hspace{-2mm} m_{\edbi}(x_i)
= \exp ( \inp{\pa{\theta}}{\pa{\phi}(x_{\alpha}) } +  \sum_{i \in N_{\alpha}} \inp{\va{\theta}{i}}{\phi_i } )
\propto b_{\alpha}(x_{\alpha}).
\end{align*}
Therefore, the local consistency condition Eq.~(\ref{eq:localconsistency}) implies that
\begin{equation*}
  \prod_{\beta \in N_i} m_{\edbi}(x_i) \propto \int \Psi_{\alpha} 
\prod_{j \in \alpha} \prod_{\beta \in N_j \smallsetminus \alpha} m_{\edbj}(x_j) {\rm d} \nu_{\alpha \smallsetminus i}.
\end{equation*}
This is equivalent to the LBP fixed point equation.
\end{AppProof}

\begin{AppProof}{Theorem \ref{thm:det1}}
The following proof proceeds in an analogous manner with Theorem 3 in \cite{STzeta1}. 
First define a differential operator 
\begin{equation*}
 \mathcal{H}:=
\sum_{\etea}\sum_{a_{e}, a_{e'}} (u_{\etea})_{a_{e}, a_{e'}}\pd{}{ (u_{\etea})_{a_{e}, a_{e'}} }
\end{equation*}
where $(u_{\etea})_{a_{e}, a_{e'}}$ denotes the $(a_{e}, a_{e'})$ element of the matrix $u_{\etea}$.
If we apply this operator to a $k$ product of $u$ terms, it is multiplied by $k$.
Since  $\log \zeta_{H}(\bs{0})=0$ and $\log \det (I -\mathcal{M}(\bs{0}))^{-1}=0$,
it is enough to prove that $\mathcal{H}\log \zeta_{H}(\bsu)=\mathcal{H} \log \det (I -\matmu)^{-1}$.
Using equations $\log \det X = \tr \log X$ and 
$- \log (1-x)=\sum_{k \geq 1} \frac{1}{k}x^k$, we have
\begin{align}
 \mathcal{H}\log \zeta_{H}(\bsu)
&=\mathcal{H} \sum_{ \mathfrak{p} \in \mathfrak{P}_H }
- \log \det ( I - \pi(\mathfrak{p}) )  \nonumber \\ 
&=\mathcal{H} \sum_{ \mathfrak{p} \in \mathfrak{P}_H }
\sum_{k \geq 1} \frac{1}{k} \tr ( \pi(\mathfrak{p})^{k}) \label{eq:thm:det3}\\
&=\sum_{ \mathfrak{p} \in \mathfrak{P}_H } 
\sum_{k \geq 1}  |\mathfrak{p}| \tr ( \pi(\mathfrak{p})^{k}) \label{eq:thm:det4}\\
&=\sum_{C: \text{closed geodesic} } \tr ( \pi(C)) 
\quad = \sum_{k \geq 1}  \tr ( \matmu^{k}). \nonumber
\end{align}
From Eq.~(\ref{eq:thm:det3}) to Eq.~(\ref{eq:thm:det4}), notice that 
$\mathcal{H}$ acts as a multiplication of $k|\mathfrak{p}|$ for each summand. 
This is because the summand is a sum of degree  $k|\mathfrak{p}|$ terms counting each $(u_{\etea})_{a_{e}, a_{e'}}$ degree one.

On the other hand, one easily observes that
\begin{align*}
 \mathcal{H} \log \det (I -\matmu)^{-1}
&= \mathcal{H} \sum_{k \geq 1} \frac{1}{k} \tr ( \matmu^{k}) \\
&= \sum_{k \geq 1}  \tr ( \matmu^{k}).
\end{align*}
Thus, the proof is completed.
\end{AppProof}

\begin{AppProof}{Theorem \ref{thm:Ihara}}

The proof is based on the decomposition in the following lemma and determinant manipulations. %
We define a linear operator by
\begin{align*}
\mathcal{T}: \vfv \rightarrow \vfe,  \qquad
&(\mathcal{T}g)(e):=g(t(e)) \\
\end{align*}
The vector spaces $\vfe$ and $\vfv$ have inner products naturally.
We can think of the adjoint of $\mathcal{T}$ which is given by
\begin{equation*}
 \mathcal{T}^{*}: \vfe \rightarrow \vfv,  \qquad
(\mathcal{T}^{*}f)(i):= \sum_{e: t(e)=i} f(e).
\end{equation*}

These linear operators have the following relation.

\noindent
{\bf Lemma}
\label{lem:decM}
\begin{equation*}
 \matmu= \bs{\iota}(\bs{u})\mathcal{T}\mathcal{T}^* - \bs{\iota} (\bs{u})
\end{equation*}
\begin{proof}[Proof of Lemma]
Let $f \in \vfv$.
\begin{align*}
\Big( \bs{\iota}(\bs{u})\mathcal{T}\mathcal{T}^* - \bs{\iota} (\bs{u}) \Big) f(e) 
&= \sum_{e': {s(e')=s(e) \atop t(e') \neq t(e)}} u^{s(e)}_{\ed{t(e')}{t(e)}}
\sum_{e'': t(e'')=t(e')}f(e'')
- \sum_{e'': {s(e'')=s(e) \atop t(e'') \neq t(e)}} u^{s(e)}_{\ed{t(e'')}{t(e)}} f(e'') \\
&= \sum_{e': {s(e')=s(e) \atop t(e') \neq t(e)}} u^{s(e)}_{\ed{t(e')}{t(e)}}
\sum_{e'': {t(e'')=t(e') \atop e''\neq e'}}f(e'') \\
&=(\matmu f)(e).
\end{align*}
\end{proof}

Using this lemma, we have
 \begin{align*}
  \zeta_{G}(\bs{u})^{-1}&= \det(I- \matmu)\\
  &=\det(I-   \bs{\iota}(\bs{u})\mathcal{T}\mathcal{T}^* + \bs{\iota} (\bs{u})) \\
  &=\det(I-   \bs{\iota}(\bs{u})\mathcal{T}\mathcal{T}^* (I+\bs{\iota} (\bs{u}))^{-1})
   \det( (I+\bs{\iota} (\bs{u})) ) \\
  &=\det(I_{r_V}- \mathcal{T}^* (I+\bs{\iota} (\bs{u}))^{-1} \bs{\iota}(\bs{u})\mathcal{T} )
  \prod_{\alpha \in F} \det(U_{\alpha})
 \end{align*}
It is easy to see that
$I_{r_V}- \mathcal{T}^* (I+\bs{\iota} (\bs{u}))^{-1} \bs{\iota}(\bs{u})\mathcal{T}
=I_{r_V}- \mathcal{T}^*\mathcal{T}+ \mathcal{T}^* (I+\bs{\iota} (\bs{u}))^{-1} \mathcal{T}$.
We also see that
\begin{equation*}
 (\mathcal{T}^*\mathcal{T}g)(i)= \sum_{e: t(e)=i} g(t(e))=d_i g(i)
\end{equation*}
and
\begin{equation*}
 (\mathcal{T}^*(I+\bs{\iota} (\bs{u}))^{-1} \mathcal{T}g)(i)
=\sum_{e: t(e)=i} ((I+\bs{\iota} (\bs{u}))^{-1} \mathcal{T}g)(e)
=(\mathcal{W}g)(i).
\end{equation*}
\end{AppProof}

\begin{AppProof}{Proposition \ref{prop:specradiusbound}}
The right inequality is obvious.
We prove the left inequality.
Let $C=\specr{  \mathcal{M}(\norm{\bsu}) } $.
It is enough to prove that $\det(I-z \mathcal{M}(\bsu))$ has no root in 
$ \{ \lambda \in \mathbb{C} |\hspace{2mm} |\lambda| < C^{-1} \}$.
Accordingly, we show that $\zeta_{H}(z \bsu)$ has no pole in the set.
Let $\mathfrak{p}$ be a prime cycle and
let $\lambda_1,\ldots,\lambda_{r}$ be the eigenvalues of $\pi(\mathfrak{p}; \bsu)$.
Then we obtain $\max |\lambda_l| \leq  \pi(\mathfrak{p}; \norm{\bsu})$.
Therefore, if $|z |< \pi(\mathfrak{p}; \norm{\bsu})^{-1} $, we have
\begin{equation}
\left|  \det(I- z^{|\mathfrak{p}|} \pi(\mathfrak{p}; \bsu)) \right|
= \left| \prod_l (1-z^{|\mathfrak{p}|} \lambda_l) \right|
\geq \Big( 1-|z|^{|\mathfrak{p}|} \pi(\mathfrak{p}; \norm{\bsu}) \Big)^{r}.  \nonumber
\end{equation}
It is not difficult to see that, for arbitrary prime cycle $\mathfrak{p}$,
an inequality
$C^{-1} \leq \pi(\mathfrak{p}; \norm{\bsu})^{-1}$ holds.
Therefore, if $|z| < C^{-1}$,
\begin{align*}
 \left| \zeta_{H}(z \bs{u}) \right|
= \left| \prod_{ \mathfrak{p} \in P } \det(I- z^{|\mathfrak{p}|} \pi(\mathfrak{p}; \bs{u}))^{-1}  \right|
\leq   \prod_{ \mathfrak{p} \in P }  \Big( 1-|z|^{|\mathfrak{p}|} \pi(\mathfrak{p}; \norm{\bsu}) \Big)^{-r}
=\zeta_{H}(|z| \norm{\bsu})^{r} < \infty.
\end{align*}
\end{AppProof}

\begin{AppProof}{Theorem \ref{thm:convexcondition} (ii) : Multinomial case}
First, we consider binary case, i.e. $\phi_{i}(x_i)=x_i \in \{ \pm 1\}$.
For $t \in [0,1]$, let us define $\eta_{ij}(t)= \E{b_{\alpha}}{x_i x_j} =t$ and $\eta_{i}(t)=0$.
Accordingly, $u^{\alpha}_{\edij}=t$ and $\bseta(t) \in L$.
As $t \nearrow 1$, $\bseta(t)$ approaches to a boundary point of $L$.
Using Theorem \ref{thm:Hashimoto}, analogous to the fixed-mean Gaussian case,
we see that $\det(\nabla^2 F(t))$ becomes negative as $t \rightarrow 1$ if  $n(H) > 1$.
Therefore, $F$ is not convex on $L$.

For general multinomial \ifas, the non convexity of $F$ is deduced from the binary case.
There is a face of (the closure of) $L$ 
that is identified with the set of pseudomarginals of the binary \ifa on the same hypergraph.
Since $0 \log 0 =0$, 
we see that the restriction of $F$ on the face is the \Bfe function of the binary \ifa.
Since this restriction is not convex, $F$ is not convex.
\end{AppProof}

\vskip 0.2in
\bibliography{BP,saved2010_BPandZetaJLMR}

\begin{thebibliography}{55}
\expandafter\ifx\csname natexlab\endcsname\relax\def\natexlab#1{#1}\fi
\expandafter\ifx\csname url\endcsname\relax
  \def\url#1{{\tt #1}}\fi


\bibitem[An(1988)]{Anote}
G.~An.
\newblock {A note on the cluster variation method}.
\newblock {\em Journal of Statistical Physics}, 52\penalty0 (3):\penalty0
  727--734, 1988.

\bibitem[Baron et~al.(2010)Baron, Sarvotham, and Baraniuk]{Baron2008}
D.~Baron, S.~Sarvotham, and R.G. Baraniuk.
\newblock {Bayesian compressive sensing via belief propagation}.
\newblock {\em Signal Processing, IEEE Transactions on}, 58\penalty0
  (1):\penalty0 269--280, 2010.

\bibitem[Bartholdi(1999)]{Bcounting}
L.~Bartholdi.
\newblock {Counting paths in graphs}.
\newblock {\em Enseign. Math., II. S{\'e}r.}, 45\penalty0 (1-2):\penalty0
  83--131, 1999.

\bibitem[Bass(1992)]{Bass}
H.~Bass.
\newblock {The Ihara-Selberg zeta function of a tree lattice}.
\newblock {\em Internat. J. Math}, 3\penalty0 (6):\penalty0 717--797, 1992.

\bibitem[Bethe(1935)]{Bethe}
H.A. Bethe.
\newblock {Statistical theory of superlattices}.
\newblock {\em Proc. R. Soc. Lon. A}, 150\penalty0 (871):\penalty0 552--575,
  1935.

\bibitem[Foata and Zeilberger(1999)]{FZcombinatorial}
D.~Foata and D.~Zeilberger.
\newblock {A combinatorial proof of Bass's evaluations of the Ihara-Selberg
  zeta function for graphs}.
\newblock {\em Transactions of the American Mathematical Society}, 351\penalty0
  (6):\penalty0 2257--2274, 1999.

\bibitem[Guckenheimer and Holmes(1990)]{GHnonlinear}
J.~Guckenheimer and P.~Holmes.
\newblock {\em {Nonlinear oscillations, dynamical systems, and bifurcations of
  vector fields}}.
\newblock Springer, 1990.

\bibitem[Hashimoto(1989)]{Hpadic}
K.~Hashimoto.
\newblock {Zeta functions of finite graphs and representations of p-adic
  groups}.
\newblock {\em Automorphic forms and geometry of arithmetic varieties},
  15:\penalty0 211--280, 1989.

\bibitem[Hashimoto(1990)]{Hzeta}
K.~Hashimoto.
\newblock {On zeta and L-functions of finite graphs}.
\newblock {\em Internat. J. Math}, 1\penalty0 (4):\penalty0 381--396, 1990.

\bibitem[Heskes(2002)]{Hstable}
T.~Heskes.
\newblock {Stable fixed points of loopy belief propagation are minima of the
  Bethe free energy}.
\newblock {\em Advances in Neural Information Processing Systems, 15}, pages
  343--350, 2002.

\bibitem[Heskes(2004)]{Huniquness}
T.~Heskes.
\newblock {On the uniqueness of loopy belief propagation fixed points}.
\newblock {\em Neural Computation}, 16\penalty0 (11):\penalty0 2379--2413,
  2004.

\bibitem[Horn and Johnson(1990)]{HJmatrix}
R.A. Horn and C.R. Johnson.
\newblock {\em {Matrix analysis}}.
\newblock Cambridge University Press, 1990.

\bibitem[Horton et~al.(2008)Horton, Stark, and Terras]{HSTweighted}
M.D. Horton, H.M. Stark, and A.A. Terras.
\newblock {Zeta Functions of weighted graphs and covering graphs}.
\newblock {\em Analysis on Graphs and Its Applications}, 77:\penalty0 29, 2008.

\bibitem[Ihara(1966)]{Idiscrete}
Y.~Ihara.
\newblock {On discrete subgroups of the two by two projective linear group over
  p-adic fields}.
\newblock {\em Journal of the Mathematical Society of Japan}, 18\penalty0
  (3):\penalty0 219--235, 1966.

\bibitem[Ihler et~al.(2005)Ihler, {Fisher III}, Moses, and
  Willsky]{ihler2005nonparametric}
A.T. Ihler, J.W. {Fisher III}, R.L. Moses, and A.S. Willsky.
\newblock {Nonparametric belief propagation for self-localization of sensor
  networks}.
\newblock {\em Selected Areas in Communications, IEEE Journal on}, 23\penalty0
  (4):\penalty0 809--819, 2005.
\newblock ISSN 0733-8716.

\bibitem[Ihler et~al.(2006)Ihler, {Fisher III}, and Willsky]{IFW}
A.T. Ihler, J.W. {Fisher III}, and A.S. Willsky.
\newblock {Loopy belief propagation: Convergence and effects of message
  errors}.
\newblock {\em Journal of Machine Learning Research}, 6\penalty0 (1):\penalty0
  905--936, 2006.

\bibitem[Ikeda et~al.(2004)Ikeda, Tanaka, and Amari]{ITAinfo}
S.~Ikeda, T.~Tanaka, and S.~Amari.
\newblock {Information geometry of turbo and low-density parity-check codes}.
\newblock {\em IEEE Transactions on Information Theory}, 50\penalty0
  (6):\penalty0 1097--1114, 2004.

\bibitem[Iwao(2006)]{Ibartholdi}
S.~Iwao.
\newblock {Bartholdi zeta functions for hypergraphs}.
\newblock {\em The Electronic Journal of Combinatorics}, 14\penalty0
  (1):\penalty0 N2, 2006.

\bibitem[Johnson et~al.(2006)Johnson, Malioutov, and Willsky]{JMWwalk}
J.~Johnson, D.~Malioutov, and A.~Willsky.
\newblock {Walk-sum interpretation and analysis of Gaussian belief
  propagation}.
\newblock {\em Advances in Neural Information Processing Systems}, 18:\penalty0
  579, 2006.

\bibitem[Jordan(1998)]{Jlearning}
M.I. Jordan.
\newblock {\em {Learning in graphical models}}.
\newblock Kluwer Academic Publishers, 1998.

\bibitem[Kotani and Sunada(2000)]{KSzeta}
M.~Kotani and T.~Sunada.
\newblock {Zeta functions of finite graphs}.
\newblock {\em Journal of Mathematical Sciences. The University of Tokyo},
  7\penalty0 (1):\penalty0 7--25, 2000.

\bibitem[Kschischang et~al.(2001)Kschischang, Frey, and Loeliger]{KFLfactor}
F.R. Kschischang, B.J. Frey, and H.A. Loeliger.
\newblock {Factor graphs and the sum-product algorithm}.
\newblock {\em IEEE Transactions on information theory}, 47\penalty0
  (2):\penalty0 498--519, 2001.

\bibitem[Malioutov et~al.(2006)Malioutov, Johnson, and Willsky]{MJWwalk}
D.M. Malioutov, J.K. Johnson, and A.S. Willsky.
\newblock {Walk-sums and belief propagation in Gaussian graphical models}.
\newblock {\em The Journal of Machine Learning Research}, 7:\penalty0 2064,
  2006.

\bibitem[Mardia et~al.(2009)Mardia, Kent, Hughes, and Taylor]{Mardia2009}
K.V. Mardia, J.T. Kent, G.~Hughes, and C.C. Taylor.
\newblock {Maximum likelihood estimation using composite likelihoods for closed
  exponential families}.
\newblock {\em Biometrika}, 96\penalty0 (4):\penalty0 975--982, 2009.

\bibitem[McEliece et~al.(1998)McEliece, MacKay, and Cheng]{Turbo}
R.J. McEliece, D.J.C. MacKay, and J.F. Cheng.
\newblock {Turbo decoding as an instance of Pearl's "belief propagation"
  algorithm}.
\newblock {\em IEEE J. Sel. Areas Commun.}, 16\penalty0 (2):\penalty0 140--52,
  1998.

\bibitem[Mezard et~al.(2002)Mezard, Parisi, and Zecchina]{MPZanalytic}
M.~Mezard, G.~Parisi, and R.~Zecchina.
\newblock {Analytic and algorithmic solution of random satisfiability
  problems}.
\newblock {\em Science}, 297\penalty0 (5582):\penalty0 812, 2002.

\bibitem[Mizuno and Sato(2004)]{MSweighted}
H.~Mizuno and I.~Sato.
\newblock {Weighted zeta functions of graphs}.
\newblock {\em Journal of Combinatorial Theory, Series B}, 91\penalty0
  (2):\penalty0 169--183, 2004.

\bibitem[Mooij and Kappen(2005)]{MKproperty}
J.M. Mooij and H.J. Kappen.
\newblock {On the properties of the Bethe approximation and loopy belief
  propagation on binary networks}.
\newblock {\em Journal of Statistical Mechanics: Theory and Experiment},
  11:\penalty0 P11012, 2005.

\bibitem[Mooij and Kappen(2007)]{MKsufficient}
J.M. Mooij and H.J. Kappen.
\newblock {Sufficient conditions for convergence of the sum-product algorithm}.
\newblock {\em IEEE Transactions on Information Theory}, 53\penalty0
  (12):\penalty0 4422--4437, 2007.

\bibitem[Murphy et~al.(1999)Murphy, Weiss, and Jordan]{MWJempiricalstudy}
K.~Murphy, Y.~Weiss, and M.I. Jordan.
\newblock {Loopy belief propagation for approximate inference: An empirical
  study}.
\newblock {\em Proc. of Uncertainty in AI}, 15:\penalty0 467--475, 1999.

\bibitem[Northshield(1998)]{Nnote}
S.~Northshield.
\newblock {A note on the zeta function of a graph}.
\newblock {\em Journal of Combinatorial Theory, Series B}, 74\penalty0
  (2):\penalty0 408--410, 1998.

\bibitem[Pakzad and Anantharam(2002)]{PAstat}
P.~Pakzad and V.~Anantharam.
\newblock {Belief propagation and statistical physics}.
\newblock {\em Conference on Information Sciences and Systems}, 2002.

\bibitem[Pearl(1988)]{Pearl}
J.~Pearl.
\newblock {\em {Probabilistic Reasoning in Intelligent Systems: Networks of
  Plausible Inference}}.
\newblock Morgan Kaufmann Publishers, San Mateo, CA, 1988.

\bibitem[Pelizzola(2005)]{Pcluster}
A.~Pelizzola.
\newblock {Cluster variation method in statistical physics and probabilistic
  graphical models}.
\newblock {\em Journal of Physics A: Mathematical General}, 38\penalty0
  (33):\penalty0 R309--R339, 2005.

\bibitem[Serre(1980)]{Strees}
J.P. Serre.
\newblock {\em {Trees}}.
\newblock Springer-Verlag, 1980.

\bibitem[Stark and Terras(1996)]{STzeta1}
H.M. Stark and A.A. Terras.
\newblock {Zeta functions of finite graphs and coverings}.
\newblock {\em Advances in Mathematics}, 121\penalty0 (1):\penalty0 124--165,
  1996.

\bibitem[Storm(2006)]{Shypergraph}
C.K. Storm.
\newblock {The zeta function of a hypergraph}.
\newblock {\em The Electronic Journal of Combinatorics}, 13\penalty0
  (R84):\penalty0 1, 2006.

\bibitem[Sudderth et~al.(2002)Sudderth, Ihler, Freeman, and
  Willsky]{Sudderth2002}
E.B. Sudderth, A.T. Ihler, W.T. Freeman, and A.S. Willsky.
\newblock {Nonparametric belief propagation and facial appearance estimation}.
\newblock In {\em IEEE International Conference on Computer Vision and Pattern
  Recognition}, pages 605--612, 2002.

\bibitem[Sunada(1986)]{SL-functions}
T.~Sunada.
\newblock {L-functions in geometry and some applications}.
\newblock {\em Lecture Notes in Math}, 1201:\penalty0 266--284, 1986.

\bibitem[Tatikonda and Jordan(2002)]{TJgibbsmeasure}
S.~Tatikonda and M.I. Jordan.
\newblock {Loopy belief propagation and Gibbs measures}.
\newblock {\em Uncertainty in AI}, 18:\penalty0 493--500, 2002.

\bibitem[Venkataraman(2009)]{Venkataraman2009}
P.~Venkataraman.
\newblock {\em {Applied optimization with MATLAB programming}}.
\newblock John Wiley and Sons, 2009.

\bibitem[Wainwright et~al.(2003{\natexlab{a}})Wainwright, Jaakkola, and
  Willsky]{Wrepara}
M.J. Wainwright, T.S. Jaakkola, and A.S. Willsky.
\newblock {Tree-based reparameterization framework for analysis of sum-product
  and related algorithms}.
\newblock {\em IEEE Transactions on Information Theory}, 49\penalty0
  (5):\penalty0 1120--1146, 2003{\natexlab{a}}.

\bibitem[Wainwright et~al.(2003{\natexlab{b}})Wainwright, Jaakkola, and
  Willsky]{WJWtree}
M.J. Wainwright, T.S. Jaakkola, and A.S. Willsky.
\newblock {Tree-reweighted belief propagation algorithms and approximate ML
  estimation by pseudomoment matching}.
\newblock In {\em Workshop on Artificial Intelligence and Statistics},
  volume~21, 2003{\natexlab{b}}.

\bibitem[Wainwright et~al.(2005)Wainwright, Jaakkola, and Willsky]{WJWMAP}
M.J. Wainwright, T.S. Jaakkola, and A.S. Willsky.
\newblock {MAP estimation via agreement on trees: message-passing and linear
  programming}.
\newblock {\em IEEE Transactions on Information Theory}, 51\penalty0
  (11):\penalty0 3697--3717, 2005.

\bibitem[Wainwright and Jordan(2003)]{WJvariational}
M.J. Wainwright and M.I. Jordan.
\newblock {Variational inference in graphical models: The view from the
  marginal polytope}.
\newblock In {\em proceedings of the annual Allerton Conference on
  Communication, Control, and Computing}, volume~41, pages 961--971, 2003.

\bibitem[Wainwright and Jordan(2008)]{WJgraphical}
M.J. Wainwright and M.I. Jordan.
\newblock {Graphical models, exponential families, and variational inference}.
\newblock {\em Foundations and Trends in Machine Learning}, 1\penalty0
  (1-2):\penalty0 1--305, 2008.

\bibitem[Watanabe and Fukumizu(2011)]{WFMathZeta}
W.~Watanabe and K.~Fukumizu.
\newblock {On the uniquness of the solution of belief propagation equation}.
\newblock {\em in preparation}, 2011.

\bibitem[Watanabe(2010)]{Wthesis}
Y~Watanabe.
\newblock {Discrete geometric analysis of message passing algorithm on graphs}.
\newblock {\em Ph.D thesis}, 2010.

\bibitem[Watanabe and Fukumizu(2009)]{WFzeta}
Y.~Watanabe and K.~Fukumizu.
\newblock {Graph zeta function in the Bethe free energy and loopy belief
  propagation}.
\newblock {\em Advances in Neural Information Processing Systems}, 2009.

\bibitem[Weiss(2000)]{W1loop}
Y.~Weiss.
\newblock {Correctness of local probability propagation in graphical models
  with loops}.
\newblock {\em Neural Computation}, 12\penalty0 (1):\penalty0 1--41, 2000.

\bibitem[Weiss et~al.(2007)Weiss, Yanover, and Meltzer]{WYM}
Y.~Weiss, C.~Yanover, and T.~Meltzer.
\newblock {MAP estimation, linear programming and belief propagation with
  convex free energies}.
\newblock {\em Uncertainty in Artificial Intelligence}, 2007.

\bibitem[Whittaker(2009)]{Wgraphical}
J.~Whittaker.
\newblock {\em {Graphical models in applied multivariate statistics}}.
\newblock Wiley Publishing, 2009.

\bibitem[Wiegerinck and Heskes(2003)]{WHfractional}
W.~Wiegerinck and T.~Heskes.
\newblock {Fractional belief propagation}.
\newblock {\em Advances in Neural Information Processing Systems}, pages
  455--462, 2003.

\bibitem[Yedidia et~al.(2001)Yedidia, Freeman, and Weiss]{YFWGBP}
J.S. Yedidia, W.T. Freeman, and Y.~Weiss.
\newblock {Generalized belief propagation}.
\newblock {\em Advances in Neural Information Processing Systems}, 13:\penalty0
  689--95, 2001.

\bibitem[Yedidia et~al.(2005)Yedidia, Freeman, and Weiss]{YFWconstructing}
J.S. Yedidia, W.T. Freeman, and Y.~Weiss.
\newblock {Constructing free-energy approximations and generalized belief
  propagation algorithms}.
\newblock {\em IEEE Transactions on Information Theory}, 51\penalty0
  (7):\penalty0 2282--2312, 2005.

\end{thebibliography}
\end{document}